 \documentclass[final,12pt]{clear2023} 

\title{Causal Inference Despite Limited Global Confounding via Mixture Models}

\usepackage[utf8]{inputenc} 
\usepackage{hyperref}       
\usepackage{url}            
\usepackage{booktabs}       
\usepackage{amsfonts}       
\usepackage{nicefrac}       
\usepackage{microtype}      
\usepackage{xcolor}         
\usepackage{ bbold }
\usepackage{enumerate}
\usepackage{bbm}

\hypersetup{colorlinks = true, linkcolor =blue, anchorcolor = red, citecolor = blue, urlcolor = blue}

\usepackage[all,cmtip]{xy}
\xyoption{frame}

\usepackage{tikz}
\usetikzlibrary{matrix, positioning}
\usepackage{listings}
\usepackage{mathtools}
\usepackage{verbatim}
\usepackage{algorithm2e}

\usepackage{enumitem}

\newtheorem{observation}[theorem]{Observation}
\newtheorem{claim}[theorem]{Claim}

\newcommand{\indep}{\perp \!\!\! \perp}

\newcommand{\rset}{\mathbb R}

\def\set#1{\{ #1 \}}
\def\abs#1{\mathopen| #1 \mathclose|}			

\def\Seq#1{\left\langle #1 \right\rangle}
\def\Set#1{\left\{ #1 \right\}}
\def\Abs#1{\left| #1 \right|}
\def\Card#1{\left| #1 \right|}
\def\Norm#1{\left\| #1 \right\|}
\def\Paren#1{\left( #1 \right)}		

\makeatletter
\def\Bigbar#1{\mathrel{\left|\vphantom{#1}\right.\n@space}}

\makeatother

\newcommand{\suppress}[1]{}

\def\eps{\varepsilon}

\def\given{\mid}

\newcommand{\be}{\begin{equation}}
		\newcommand{\ee}{\end{equation}}
\newcommand{\bea}{\begin{eqnarray}}
		\newcommand{\eea}{\end{eqnarray}}
\newcommand{\bean}{\begin{eqnarray*}}
		\newcommand{\eean}{\end{eqnarray*}}
		
\newcommand{\condV}[1]{
            \scalebox{.4}{\begin {tikzpicture}[-latex ,auto ,node distance =1.7 cm and 1.7 cm ,on grid , ultra thick, baseline={([yshift=-.5ex]current bounding box.center)}, cstate/.style ={ font = \scriptsize, circle, draw, minimum width =.8 cm, fill=black, text = white}]
             \node[cstate] (V2) {$#1$};
            \end{tikzpicture}}
}
\newcommand{\indV}[2]{
            \scalebox{.4}{\begin {tikzpicture}[-latex ,auto ,node distance =1.7 cm and 1.7 cm ,on grid , ultra thick, baseline={([yshift=-.5ex]current bounding box.center)}, state/.style ={ font = \scriptsize, circle, draw, minimum width =.5 cm, ultra thick}]
             \node[state, fill = #2!50] (V2) {$#1$};
            \end{tikzpicture}}
}

\DeclareMathOperator{\Pa}{\operatorname{\textbf{PA}}}
\DeclareMathOperator{\Bottom}{\operatorname{\textbf{BOT}}}
\DeclareMathOperator{\An}{\operatorname{\textbf{AN}}}
\DeclareMathOperator{\AV}{\operatorname{\textbf{AV}}}
\DeclareMathOperator{\OH}{\operatorname{\textbf{OH}}}
\DeclareMathOperator{\Ch}{\operatorname{\textbf{CH}}}
\DeclareMathOperator{\Mb}{\operatorname{\textbf{MB}}}
\DeclareMathOperator{\mb}{\operatorname{\textbf{mb}}}
\DeclareMathOperator{\pa}{\operatorname{\textbf{pa}}}
\DeclareMathOperator{\ch}{\operatorname{\textbf{ch}}}
\DeclareMathOperator{\Cond}{\operatorname{\textbf{COND}}}
\DeclareMathOperator{\MCond}{\operatorname{\textbf{COND}}}
\DeclareMathOperator{\cond}{\operatorname{\textbf{cond}}}
\DeclareMathOperator{\mcond}{\operatorname{\textbf{cond}}}

\def\yone{y^1}

\def\II{\vec{I}}

\def\XX{\mathcal{X}}
\def\VV{\mathcal{V}}

\def\Vars{\vec{V}}
\def\Centers{\vec{X}}
\def\PP{\mathcal{P}}
\def\ePP{\widetilde{\mathcal{P}}}
\DeclareMathOperator{\Top}{\textbf{TOP}}
\DeclareMathOperator{\ttop}{\textbf{top}}
\def\DEFAULTS{\mathsf{Defaults}}

\newcommand{\din}{\Delta_{\text{in}}}
\newcommand{\dout}{\Delta_{\text{out}}}

\def\calA{\mathcal{A}}

\def\G{\mathcal{G}}

\def\eps{\varepsilon}
\def\given{\mid}

\usepackage{times}

\clearauthor{%
 \Name{Spencer Gordon} \Email{slgordon@caltech.edu}\\
 \addr California Institute of Technology
 \AND
 \Name{Bijan Mazaheri} \Email{bmazaher@caltech.edu}\\
 \addr California Institute of Technology
 \AND
 \Name{Yuval Rabani} \Email{yrabani@cs.huji.ac.il}\\
 \addr Hebrew University in Jerusalem
 \AND
 \Name{Leonard Schulman} \Email{schulman@caltech.edu}\\
 \addr California Institute of Technology
}

\begin{document}

\maketitle

\begin{abstract}%
  A Bayesian Network is a directed acyclic graph (DAG) on a set of $n$ random variables (the vertices); a Bayesian Network Distribution (BND) is a probability distribution on the random variables that is Markovian on the graph. A finite $k$-mixture of such models is graphically represented by a larger graph which has an additional ``hidden'' (or ``latent'') random variable $U$, ranging in $\{1,\ldots,k\}$, and a directed edge from $U$ to every other vertex. Models of this type are fundamental to causal inference, where $U$ models an unobserved confounding effect of multiple populations, obscuring the causal relationships in the observable DAG. By solving the mixture problem and recovering the joint probability distribution with $U$, traditionally unidentifiable causal relationships become identifiable. Using a reduction to the more well-studied ``product'' case on empty graphs, we give the first algorithm to learn mixtures of non-empty DAGs. 
\end{abstract}


\section{Introduction}

The distinction between causation and correlation is paramount to the development of scientific knowledge. While learning spurious correlations is sufficient for \emph{predicting} outcomes, causal inference seeks the effect of intervening on a system. This fortification is an essential step towards using data to \emph{recommend} courses of action, whether they are medical treatments or a changes in economic policy.

Interventional effects are often explored in experimental settings via ``Randomized Controlled Trials'' (RCTs), which decouple intervened variables from potential confounding.  Unfortunately,  experiments are impossible or prohibitively expensive in a wide range of complex natural and artificial systems. In such settings
researchers often have to resort to collecting samples of an assortment of measurements from a joint distribution governed by causal relations.  Such relations define a structural causal model (SCM), described by a Bayesian network \citep{pearl1985bayesian}.

The assumption of causal sufficiency describes the often unrealistic case in which all relevant variables in the causal system are observed.  Under causal sufficiency, the consequences of all interventions on the system are identifiable by adjusting for the confounders described by the SCM.  If causal sufficiency is relaxed, unobserved confounders cannot be adjusted for using observed statistics, limiting the identifiability of some interventional distributions. The \emph{graphical} conditions under which causal relationships are identifiable are well studied~\citep{shpitserP06,huangV08,Pea09,SpGlS00,PetersJS17,KF09}. There has been little exploration into the identifiability of causal relationships using \emph{numerical} conditions. 

This paper will address a setting in which an unobserved multiplicity of populations that confounds an entire system - often called a \emph{mixture model}. The setting emerges when combining data from many setting or laboratories, whose environments may exert a confounding influence. Adjusting for the observed confounder in such settings would involve considering each data source separately.  To make full use of the merged dataset, we may want to instead assume that some simpler unobserved $U$ (taking on fewer states than the number of datasets) is sufficient to control for the confounding induced by the merging.  Examples of such a $U$ could include private attributes like health status or artificially constructed classes such as individual price sensitivity.  

Such a setting satisfies neither observability of the ``backdoor adjustment set'' nor the sparsity requirements on the connection of $U$ for the ``frontdoor criterion'' \citep{Pea09}.  As a result,  the causal relationships within the SCM are considered unidentifiable in the traditional framework.  This worst case analysis ignores the fact that limited cardinality of $U$ would also limit $U$'s ability to completely hide the causal dynamics.  We will show that mild additional assumptions in conjunction with a known DAG structure allow identification of the joint probability distribution with the previously unobservable $U$, a sufficient condition for applying the backdoor adjustment.

\subsection{Problem Statement}
A Bayesian network is a directed acyclic graph $\G=(\Vars,\vec{E})$, on a set of $\abs{\Vars} = n$ random variables. A corresponding Bayesian network distribution (BND) is a probability distribution on the random variables that is Markovian on the graph. That is to say, the joint distribution on the variables can be factored as $\prod_{i=1}^n \PP[V_i=v_i\mid \pa(V_i)]$ where $\pa(V_i)$ is the assignment to the parents of $V_i$.
A $k$-MixBND on $\G$ is a convex combination, or ``mixture'', of $k$ BNDs.  We represent this situation graphically by a single unobservable random variable $U$ with edges to each of the variable $V \in G$. Here, $U$ is referred to as a ``source'' variable with range ${1, \ldots, k}$ and the variables in $\G$ are referred to as the ``observables.'' The main complexity parameter of the problem is $k$, representing the number of mixture constituents or ``sources.''

\suppress{Models of this type (and even more general) are fundamental to research in causal inference, where the chief difficulty has to do with performing statistical inference for collections of random variables which, besides some known potential direct effects on each other (the directed edges), are also affected by unseen confounders.}

The extremely special case where $\G$ is empty has been of longstanding interest in the theory literature.  Such a distribution is a mixture of $k$ product distributions or $k$-MixProd. See Fig.~\ref{small-examples}.
   \begin{figure}[h]
    \centering
    (a)
    \scalebox{.45}{
        \begin {tikzpicture}[-latex ,auto ,node distance =2 cm and 2 cm ,on grid , semithick, state/.style ={ circle, draw, minimum width =1 cm}]
        \node[state,] (X1) at (-3,0) {$V_1$};
         \node[state,] (X2) at (-1,0) {$V_2$};
          \node[state,] (X3) at (1,0) {$V_3$};
         \node[state,] (X4) at (3,0) {$V_4$};
        \node[state, dashed] (U) at (0 , 2) {$U$};
		\path[ultra thick] (X1) edge[bend right = 35] (X3);
        \path[ultra thick] (X1) edge (X2) (X2) edge (X3) (X3) edge (X4);
        \path[ultra thick, dashed] (U) edge (X1) (U) edge (X2) (U) edge (X3) (U) edge (X4);
        \end{tikzpicture}}
        (b)
	 \scalebox{.45}{
        \begin {tikzpicture}[-latex ,auto ,node distance =2 cm and 2 cm ,on grid , semithick, state/.style ={ circle, draw, minimum width =1 cm}]
        \node[state,] (X1) at (-3,0) {$V_1$};
         \node[state,] (X2) at (-1,0) {$V_2$};
          \node[state,] (X3) at (1,0) {$V_3$};
         \node[state,] (X4) at (3,0) {$V_4$};
        \node[state, dashed] (U) at (0 , 2) {$U$};
		\path[draw=none] (X1) edge[bend right = 35, draw=none] (X3);
        \path[ultra thick, dashed] (U) edge (X1) (U) edge (X2) (U) edge (X3) (U) edge (X4);
        \end{tikzpicture}}
    \caption{(a) A small Bayesian Network, with latent variable $U$. (b) In the empty graph, a $k$-MixBND is a $k$-MixProd (mixture of product distributions).}  \label{small-examples} 
  \end{figure}
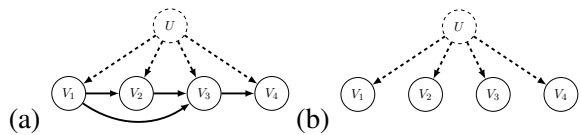

In this paper we study the \emph{identification} problem for $k$-MixBNDs. Specifically, given the graph $\G$, and given a joint distribution $\PP$ on the variables (vertices), recover up to small statistical error (a) the mixture weights (probability of each source), up to a permutation of
the constituents, and (b) for every mixture source and for every vertex $V$, its conditional distribution
given each possible setting to the parents of $V$.  This task identifies the joint probability distribution $\PP(U, \vec{V})$ up to the $k!$ permutations in the label $U$. Identification will be shown by giving an algorithm that reduces the $k$-MixBND problem into a series of calls to a $k$-MixProd oracle. $k$-MixBND models are not always identifiable, as further discussed in \emph{Assumptions} below. 
Thus,
another contribution of our paper is to establish a sufficient setting to guarantee identifiability.

\paragraph{Assumptions}
The following assumptions are used throughout this paper. 
\begin{enumerate}[nosep, leftmargin=.4cm]
\item 
\label{as:discrete}
\emph{We have access to a $k$-MixProd oracle requiring $\mathcal{O}(k)$ variables that are independent within each source. }
As different algorithms have different requirements for the number of independent variables, we will keep our results agnostic to these requirements. The most efficient published algorithm is given in \citet{gordon2021source}, which requires $N_{\text{mp}} = 3k-3$ variables and time complexity $\exp(k^2)$. Recent unpublished work improves the complexity bound to $\exp(k \log k)$. 

\item 
\label{as:discrete}
\emph{The observable variables in our BND are binary and discrete.}
While a number of papers have focused on continuous or large-alphabet settings, we restrict our focus to the simplest setting of binary, discrete variables. Appendix~\ref{apx: large alphabets for DAGs} gives a reduction from alphabets of any size $d$ to the binary case, incurring a mild cost in complexity.

\item \label{as:k}
  \emph{The mixture is supported on $\leq k$ sources.}
If the hidden variable $U$ has unrestricted range (Specifically, range $k=2^n$ would be enough),  the model is rich enough to describe \emph{any} probability distribution on $\vec{V}$,  making identification impossible. The question is therefore one of trading $k$ against the sample and computational complexity of an algorithm (and the degree of the network).


\item \label{as:types_of_networks}
  \emph{The underlying Bayesian DAG is sufficiently sparse.}
In order to reduce $k$-MixBND to $k$-MixProd we need sufficiently many variables that can be separated from each other by conditioning on disjoint \emph{Markov boundaries} (example in Fig.~\ref{fig:disjointmb}, definition in Sec.~\ref{sec:bayes_net_basics}).  As a result, the complexity of the algorithm is exponential in the size of a Markov boundary.  Both for complexity and in order to keep $n$ small, a bound on the maximum degree $\Delta$ is required.  We require $n \geq (\Delta+1)^4 N_{\text{mp}}$.\footnote{If the skeleton of $\G$ happens to be a path,
then we only need a milder condition that $n\ge 2 N_{\text{mp}} $. For details see Appendix~\ref{apx:paths}.} 

\begin{figure}[h]
    \centering
   \begin{tikzpicture}[-latex ,auto ,node distance = 4.8 cm and 4.8 cm ,on grid , ultra thick]
             \node (kmixbnd) {
                \scalebox{.42}{\begin {tikzpicture}[-latex ,auto ,node distance =4 cm and 4 cm ,on grid , ultra thick, state/.style ={font = \LARGE, circle, draw, minimum width =.5 cm, ultra thick}, cstate/.style ={ font = \LARGE, circle, draw, minimum width =.5 cm, fill=black, text = white}]
                    \filldraw[color=red, fill=red, fill opacity=.11, ultra thick](5.1 , 2.9) rectangle (14.9, 1.1);
                    \filldraw[ultra thick, color = green!60!black, fill=green, fill opacity=.11](11.1,-.9) -- (11.1, .9) -- (17.1, .9) -- (17.1, 2.9) -- (18.9, 2.9) -- (18.9, .9) -- (20.9, .9)-- (20.9, - .9) -- cycle;
                    \filldraw[color=orange, fill=orange, fill opacity=.11, ultra thick](23.1,2.9) rectangle (24.9, -.9);
                    \filldraw[color=blue, fill=blue, fill opacity=.11, ultra thick](-.9,2.9) rectangle (4.9, -.9);
                    \node[state, fill=blue!40] (X1) at (0,0) {$V_1$};
                    \node[cstate] (V1) at (2,2) {$V_2$};
                    \node[cstate] (V2) at (4,0) {$V_3$};
                    \node[cstate] (V3) at (6,2) {$V_4$};
                    \node[state] (V4) at (8,0) {$V_5$};
                    \node[state, fill=red!50] (X2) at (10,2) {$V_6$};
                    \node[cstate] (V5) at (12,0) {$V_7$};
                    \node[cstate] (V6) at (14,2) {$V_8$};
                    \node[state, fill=green!50] (X3) at (16,0) {$V_9$};
                    \node[cstate] (V7) at (18,2) {$V_{10}$};
                    \node[cstate] (V8) at (20,0) {$V_{11}$};
                    \node[cstate] (V9) at (24,0) {$V_{12}$};
                    \node[state, fill=orange!50] (X4) at (24,2) {$V_{13}$};
                    \node[state, dashed] (U) at (12,  6) {$U$};
                    \path[ultra thick] (X1) edge (V2) (V1) edge (V2) (V1) edge (V3) (V2) edge (V4) (V3) edge (X2) (X2) edge (V6) (V4) edge (V5) (V5) edge (X3) (X3) edge (V7) (X3) edge (V8) (V8) edge (V9) (V9) edge (X4);
                    \path[ultra thick, dashed] (U) edge[bend right = 38] (X1) (U) edge[bend right = 25] (V1) (U) edge[bend right = 35] (V2) (U) edge[bend right = 35] (V2) (U) edge[bend right = 20] (V3) (U) edge[bend right = 25] (V4) (U) edge[bend right = 15] (X2) (U) edge[] (V5) (U) edge[bend left = 20] (V6) (U) edge[bend left = 25] (X3) (U) edge[bend left = 20] (V7) (U) edge[bend left = 35] (V8) (U) edge[bend left = 35] (V9) (U) edge[bend left = 35] (X4);
                    \end{tikzpicture}}};
                \node (kmixprod) [below =of kmixbnd] {
                    \scalebox{.5}{\begin {tikzpicture}[-latex ,auto ,node distance =4 cm and 4 cm ,on grid , ultra thick, font=\LARGE, state/.style ={font=\LARGE, circle, draw, minimum width =1 cm}, cstate/.style ={circle, draw, minimum width =.5 cm, fill=black, text = white}]
                    \node[draw, dashed, ultra thick, font = \LARGE] (U) at (7.5,  3) {$U \given \Cond$};
                    \node[draw=blue, fill=blue!5, ultra thick, font = \LARGE] (V1) at (0,0) {$V_1 \given V_2, V_3$};
                    \node[draw=red, fill=red!5, ultra thick,font = \LARGE] (V6) at (5,0) {$V_6 \given V_4, V_8$};
                    \node[draw=green!50!black, fill=green!5, ultra thick, font = \LARGE] (V9) at (10,0) {$V_9 \given V_7, V_{10}, V_{11}$};
                    \node[draw=orange, fill=orange!5, ultra thick, font = \LARGE] (V13) at (15,0) {$V_{13} \given V_{12}$};
                    \path[dashed, ultra thick] (U) edge[bend right = 15] (V1) (U) edge[bend right = 15] (V6) (U) edge[bend left = 15] (V9) (U) edge[bend left = 15] (V13);
                    \end{tikzpicture}}};
                \node (cond) at (-2.5,-3) {
                $\Cond := \left\{\begin{aligned} &\condV{V_2}, \condV{V_3}, \condV{V_4}, \condV{V_7},\\ &\condV{V_8}, \condV{V_{10}}, \condV{V_{11}}, \condV{V_{12}}
                \end{aligned}\right\}$};
                \node (ind) at (+2,-3) {$\vec{I} := \left\{ \indV{V_1}{blue}, \indV{V_6}{red}, \indV{V_9}{green}, \indV{V_{13}}{orange} \right\}$};
                \path[line width=1mm] (kmixbnd) edge (kmixprod);
                \node (condlabel) [above = .8cm of cond] {"Conditioning set"};
                \node (indlabel) [above = .8cm of ind] {"Independent set"};
         \end{tikzpicture}
    \caption{
        The reduction process of conditioning on $\Cond$ to create an instance of $k$-MixProd.  A Bayesian network with four vertices $V_1, V_6, V_9, V_{13}$ and their corresponding disjoint Markov boundaries are indicated. } \label{fig:disjointmb}
\end{figure}
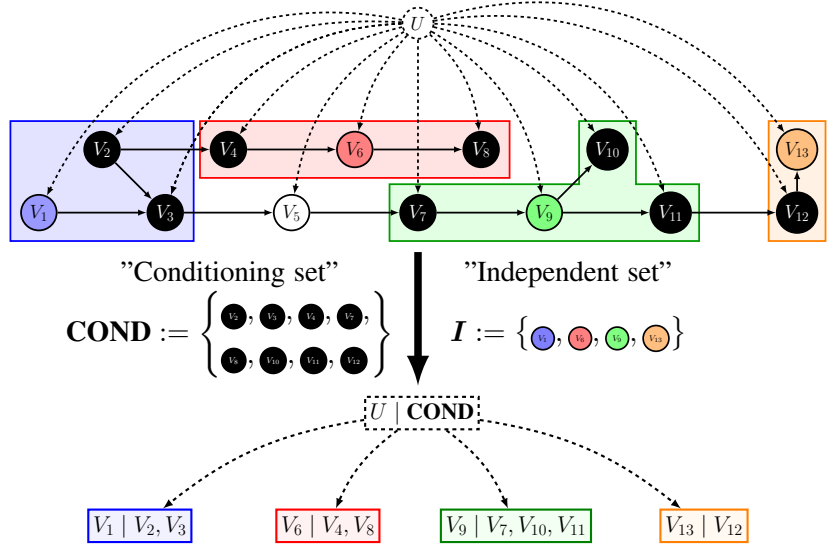

\item
\label{as:separated}
  \emph{The resulting product mixtures are non-degenerate.}
Even in mixtures of graphs with sparse structure (in particular the empty graph---the $k$-MixProd problem), the $k$-MixBND can be unidentifiable if the mixture components are insufficiently distinct. (E.g., trivially, a mixture of identical sources generates the same statistics as a single source.) Past work has used conditions such as $\zeta$-separation in \citet{gordon2021source} to ensure that 
matrices representing the parameters for each source are well-conditioned. These are not always necessary conditions; characterizing necessary conditions is a difficult question tackled in part in \citet{gordon2022hadamard}. 

\item 
\label{as:known_dag}
\emph{The DAG structure representing conditional independence properties within each source, or a common supergraph of these structures,  is known.}
It is often the case that domain knowledge provides an understanding of the causal DAG.  If the causal DAG is unknown, we are faced with a different problem commonly known as ``Causal Discovery''(see~\citet{10.3389/fgene.2019.00524} for a recent survey.) A method for causal discovery in the presence of a universal confounder has been suggested in \citet{anandkumar2012learning} by substituting independence tests with rank tests.  In our motivating example of dataset merging, it is likely that the structure can be learned from an individual dataset.  Like in \citet{anandkumar2012learning}, the presented algorithm will actually only require a supergraph of the true structure.  Hence,  some uncertainty in knowledge of the graph can be tolerated. In fact, the algorithm also works even if the different components of the $k$-MixBND use slightly different causal graphs.
\end{enumerate}

\subsection{Summary of contributions}\label{sec:summary of contributions}
\begin{theorem} Our algorithm identifies a $k$-MixBND distribution with on a graph of maximum degree $\Delta$ and of size $n \geq \Omega(N_{\text{MP}} \Delta^4)$,
using $O(n2^{\Delta^2})$ calls to an oracle for the $k$-MixProd problem.  For an exact statement see Theorem~\ref{thm:main result on good collection}.\end{theorem}
The algorithm will be built on the insight that conditioning on a set of Markov boundaries $\Cond \subset \vec{V}$ of $\II \subset V$ induces \emph{within-source} independence (that is, $V_i \indep V_j \given U, \Cond$ for all $V_i, V_j \in \II$).  This describes an instance of $k$-MixProd for which we can identify the joint probability distribution $\PP(\II, U \given \Cond)$.  See Figure~\ref{fig:disjointmb} for an illustration. 

Recovering $\PP(\II, U \given \Cond)$ for some $\II \subset \vec{V}$ is insufficient to recover the full joint probability distribution $\PP(\vec{V}, U)$.  Hence, we execute a set of $O(n2^{\Delta^2})$ ``runs" of a $k$-MixProd oracle on differing $\II, \Cond$ and assignments to their Markov boundaries and synthesize information gained from these runes into the joint probability distribution.

 The first challenge is to handle symmetries in permutation of the output labels of $U$ by ``aligning'' the outcomes of these runs. The second challenge is to remove the conditioning of $\vec{C}$ from each run. We do this by synthesizing the results of many runs with a procedure we call ``Bayesian unzipping.'' Our key contributions can be summarized by these ``alignment'' and ``unzipping'' procedures, as well as the notion of a ``good collection of runs'' that allows for the successful application of these sub-processes \suppress{We also make a key contribution to the $k$-MixProd setting, harnessing our ``alignment'' framework to extend the power of the \citet{gordon2021source} algorithm to non-binary observable variables. }

\paragraph{Organization} The rest of the paper is organized as follows.  In Section~\ref{subsection:background} we outline the literature background of the problem. In Section~\ref{sec:bayes_net_basics} we give some Bayesian network notation.  In Section~\ref{sec:applying_a_k} we formally develop the notion of a ``run,'' which calls a $k$-MixProd oracle. In Section~\ref{sec:combining} we explain how the output of the ``runs'' is combined to get the desired mixture parameters.  This section details the processes of alignment and Bayesian unzipping. Section~\ref{sec:good_collection} explains what is necessary in a group of runs in order for the algorithm to succeed, which provides a framework for defining algorithms in terms of sets of runs.

In Appendix~\ref{apx: BND pseudo} we define the $k$-MixBND algorithm in pseudo-code. 
In Appendix~\ref{apx: BND analysis} we analyze the $k$-MixBND algorithm. In Appendix~\ref{apx: good collections} we prove the existence of a good collection of runs. In Appendix~\ref{apx: large alphabets for DAGs} we generalize our algorithms to handle non-binary observations. 

\subsection{Background} \label{subsection:background}
There are two problems within mixture models: (1)
\emph{Learning} the model, namely, producing any model consistent with (or close to) the observations; (2) \emph{Identifying} the model, namely, producing the true model (or one close to it) up to permutations in the source label.
The feasibility of the identification problem hinges on a one-to-one mapping between the observed statistics and the model's parameters.  When using the resulting model to deconfound causal relationships, it is imperative that that the joint probability distribution with the confounder be identified.

The idea of using mixture models to identify parameters in latent variables dates back to \citet{allman2009identifiability}, who use algebraic methods from \citet{kruskal1976more, kruskal1977three} to exploit within-source independence. \citet{anandkumar2012method} follows a similar strategy using tensor decomposition.  Both of these works rely on within-source independence between three variables with support on large alphabets (i.e. large cardinality) to achieve \emph{generic}\footnote{Here, ``generic'' identifiability means that the set of unidentifiable models makes up a Lebesgue measure $0$ subset in the parameter space.} identifiablility. Specifically, the alphabet size of the independent variables must scale linearly with $k$.  Such a requirement can be achieved by combining variables of smaller variables, i.e. two binary variables can be combined to make up an alphabet size of $4$. \citet{allman2009identifiability} showed that identifiability holds \emph{generically} for smaller alphabets with $\mathcal{O}(\log(k))$ independent variables via these combinations. 

A different line of work in the theory community seeks \emph{full} identifiability,  requiring $\mathcal{O}(k)$ independent variables in conjunction with separation conditions to ensure that the visible variables behave differently in different mixture components \footnote{Separation conditions can be thought of as a form of faithfulness of edges from $U$ to $\vec{V}$}.  This problem, referred to as $k$-MixProd, and has been studied for nearly 30
years~\citep{KMRRSS94,CGG01,FM99,FOS08,CR08,tahmasebi2018identifiability,ChenMoitra19,gordon2021source}.\footnote{We do not even try to list the extensive analogous literature for parametrized distributions over $\rset$.} In~\citet{FOS08} a seminal algorithm for $k$-MixProd was given with running
time  $n^{O(k^3)}$ for mixtures on $n$ binary variables ($n$ sufficiently large).
This was improved in~\citet{ChenMoitra19} to $k^{O(k^3)} n^{O(k^2)}$. The most recent algorithm,~\citet{gordon2021source}
 identifies a mixture of $k$ product distributions on at least $3k-3$ variables in time
$2^{O(k^2)} n^{O(k)}$, under a mild ``separation'' condition that excludes unidentifiable instances. Under somewhat stricter separation, the time complexity improves to $2^{O(k^2)} n$.  One can choose between the generic identifiability and full identifiability approaches when reducing Bayesian network mixtures to mixtures of products. The same complexity bottlenecks appear in both strategies,  but the number of independent variables that must be instantiated differs. 

To our knowledge, the only other attempt at detailing a multiple-run reduction to $k$-MixProd is \citet{anandkumar2012learning}, which gives an algorithm for mixtures of Markov random fields--i.e., undirected graphical models.  As both papers make use of boundary conditioning to induce independence and a form of ``alignment,'' our paper can be thought of as both an improvement and an extension to the directed graph case.  While \citet{anandkumar2012learning} require a single variable that is independent from the rest of the structure for allignment, our algorithm develops the notion of ``good collections of runs'' to eliminate this restriction -- a contribution which may have implications in the Markov random field setting as well. Additional complications arise for directed graphs because the outputs of the $k$-MixProd subroutine are conditioned on their Markov boundaries while the desired parameters are only conditioned on their \emph{parents}.  Finally, we note that \citet{anandkumar2012learning} only guarantees identification of second order marginal probabilities, which is insufficient for causal identification.  
\footnote{We also mention that
\citet{anandkumar2010high} and \citet{anandkumar2012learning} introduce the idea of a sparse local separator; if this can be adapted to the directed-graph case one might be able to somewhat relax assumption~\ref{as:types_of_networks}. We do not attempt this in this paper.}

\paragraph{Other related work}
\citet{kivva2021learning} contains as a special case a reduction to the $k$-MixProd problem. Their goal is to learn a causal graphical model with latent variables, but with a very different structure on the visible and latent variables. They allow for a DAG of latent variables with visible children (which is learned as part of their algorithm); on the other hand, they require that there be no causal relations between visible variables. In our work, the structure on the latent variables is trivial (since there is a single latent variable), but the structure on the visible variables is arbitrary. Characterizing identifiability in the generalization of both these settings in which we allow structure on both the visible and latent portion of the graph is a nice problem beyond the scope of this paper.

Another paper with kindred motivation to ours is~\citet{kumarS21}, which studies inference of a certain kind of MixBND, in which  the structure of the Bayesian network is known, but the data collected is a mixture over some $m$ unknown interventional distributions. The authors give sufficient conditions for identifiability of the network and of the intervention distributions. At a technical level, the papers are not closely related. $k$ is not a parameter in their work, and instead what is essential is an ``exclusion'' assumption which says that each variable has some value to which it is not assigned by any of the interventions. 

Some other loosely related work includes learning hidden Markov models~\citep{hsuKZ12,anandkumar2012method,SharanKLV17},
an incomparable line of work to our question, but with somewhat similar motivation. In the same vein, some papers study learning
mixtures of Markov chains from observations of random paths through the state space~\citet{BGK04,GKV16}. These models,
too, differ substantially from the models addressed in this paper, and pose very different challenges. 
Literature on causal structure learning~\citep{spirtes2000constructing, causationreview}
answers the question of identifying the \emph{presence} of hidden confounders.  Fast Causal Inference (FCI) harnesses
observed conditional independence to learn causal structure, which can detect the presence of unobserved variables when
the known variables are insufficient to explain the observed behavior. This literature includes the MDAG problem in which the
DAG structure may depend upon the hidden variable; see~\citet{thiesson1998} for heuristic approaches to this problem. Other
related works study causal inference in the presence of visible ``proxy'' variables which are influenced by a latent
confounder~\citep{miao2018identifying, kuroki2014measurement, mazaheri2023causal}.  This has more recently given rise to attempts at deconfounding using multiple causes~\citep{heckerman2018accounting,ranganath2018multiple,wang2018blessings}. The initial assumptions of \citet{wang2018blessings} were shown to be insufficient for deconfounding in \citet{ogburn2019comment}. This illustrates the necessity of identifying of the joint probability distribution with the confounder. \footnote{Thanks to Betsy Ogburn for her thoughts on this topic.}

Finite mixture models have been the focus of intense research for well over a century, since pioneering work in the late
1800s~\citep{Newcomb86,Pearson94}, and doing justice to the vast literature that emanated from this endeavor is impossible
within the scope of this paper. See, e.g., the surveys~\citet{Everitt1981,TSM85,lindsay1995mixture,McLachlanLR19}.

\subsection{Notation}\label{sec:bayes_net_basics}
A Bayesian network consists of a directed acyclic graph (DAG) $\G=(\Vars,\vec{E})$ and a probability distribution $\PP$ over $\Vars=\Set{V_1,\dotsc,V_n}$ factoring according to $\G$, i.e., \[\PP(V_1,V_2,\dotsc, V_n) = \prod_{i=1}^n \PP(V_i\given \Pa(V_i)),\] where $\Pa(V)$ is the set of parents of $V\in \Vars$. (Similarly, let $\Ch(V)$ denote the children of $V$.) Equivalently, $\PP$ must satisfy that for any $V\in\Vars$, $V$ is independent of its non-descendants, given its parents.

Generally, we will use bold letters to denote sets of variables.  For a matrix $M$, we will use $M_{ij}$ to identify entries.  We use $M_{i,-}$ to denote the vector $(M_{i,1},  M_{i,2}, \ldots M_{i,k})$.

Conditioning on the ``Markov boundary'' of a vertex $\Mb(V)$ makes $V$ conditionally independent from everything else in the graph \citep{pearl2014probabilistic}.
\begin{definition}[Markov Boundary] \label{def:mb}
    For a vertex $Y$ in a DAG $\G=(\Vars,\vec{E})$, the \textbf{Markov boundary} of $Y$, denoted $\Mb(Y)$, is defined by
    \[ \Mb(Y) \coloneqq \Pa(Y) \cup \Ch(Y) \cup \Pa(\Ch(Y)) \setminus \{Y\}. \]
\end{definition}

\begin{lemma}[See \citet{pearl2014probabilistic}]\label{lem:mbindependent} For any vertex $V\in \Vars$ and subset $S\subseteq V\setminus (\Mb(V)\cup \Set{V})$, $\PP(V\given \Mb(V),S) = \PP(V\given \Mb(V))$.
\end{lemma}

\begin{observation} For any $X,Y \in \Vars$, $X\in \Mb(Y) \iff Y \in \Mb(X)$ \label{obs:mb mb}
\end{observation}
\paragraph{Uppercase/lowercase conventions}
Following notation in causal inference literature,  we will use lowercase letters to denote assignments.  For example
$\PP(v \given u) = \PP(V=v \given U=u)$.  Following this convention,  we will write $\pa(V)$, $\ch(V)$, and $\mb(X)$
to denote assignment to the parents $\Pa(V)$, children $\Ch(V)$, and Markov boundary $\Mb(V)$.

\paragraph{Within-source probabilities}
It will be easier to write $\PP_u(v) = \PP(v \given u)$ to give the probability distribution within a source.

Finally,  here are a few more definitions that will make the upcoming sections simpler.
\begin{definition}[Top] \label{def:top}
    We will use $\Top(V)$ to denote $\Mb(V) \setminus \Ch(V)$.
\end{definition}

\begin{definition}[Depth of a vertex]
    Given a DAG $\G=(\Vars, \vec{E})$ and any vertex $V\in \Vars$, let $d_{\G}(V)$ be the \textbf{depth} of $V$ in $\G$, i.e. the length of the shortest path from a degree-$0$ vertex in $\G$. When $\G$ is clear from context, we'll omit the subscript.
\end{definition}

\begin{definition}  We'll introduce a parameter $\gamma(G)$ which will appear in the complexity of the identification procedure, which is defined by $\gamma(G) \coloneqq \max_{V\in \Vars} \Card{\Mb(V)}$. \label{def:gamma}
\end{definition}

\section{Applying a $k$-MixProd run}\label{sec:applying_a_k}
Our algorithm will induce instances of $k$-MixProd through post-selected conditioning.  A significant portion of this paper will be accounting for multiple calls (or ``runs'') of a $k$-MixProd oracle and explaining how their results can be combined. 
\subsection{Describing runs}
We will need to keep track of two crucial elements of each ``run'' of a $k$-MixProd oracle.
\begin{enumerate}[nosep]
\item Which variables $\in \vec{V}$ are passed to our $k$-MixProd oracle as independent variables (the \textbf{independent set}). 
\item Which variables  $\in \vec{V}$ we have conditioned on (the \textbf{conditioning set}) and what values we have post-selected these variables to take. 
\end{enumerate}

A sufficient conditioning set to induce within-source independence among the independent set is the union of their Markov boundaries. This will be further refined in Subsection~\ref{sec:bottom vertices}.

\begin{definition}[Run]\label{def:run}
    A \textbf{run} over a graph $\G = (\Vars,\vec{E})$ is a tuple $a=(\II^a, f^a)$ where $\II^a \subseteq \Vars$ are variables that we will $d$-separate (within each source) by conditioning on assignments to the set
    \[ \MCond^a \coloneqq \bigcup_{I \in \II^a} \Mb(I) \]  The value of the assignment is given by $f^a: \MCond^a \to \Set{0,1}$.  We'll call $\II^a$ the \textbf{independent set} for $a$, and $\MCond^a$ the \textbf{conditioning set}.
\end{definition}
We will restrict our attention to \emph{well-formed runs},  i.e.  runs for which $\II^a \cap \MCond^a = \emptyset$. 
    
\begin{definition}
    An individual run $a=(\II^a, f^a)$ is \textbf{$N$-independent} if $\Card{\II^a} \geq N$.
\end{definition}

\paragraph{Superscript notation}
We'll write $\mb^a(V)$, $\pa^a(V)$, $\ch^a(V)$ to refer to the assignment to the Markov boundary of $V$, parents of $V$, and children of $V$ as set by run $a$.\footnote{Any quantities parameterized by a run will take the parameter as a superscript.} In a similar spirit, we'll occasionally write $v^0$ to denote the assignment  $V=0$.

\begin{definition}[Distribution induced by a run]
    For any well-formed run $a$, the induced distribution on the variables in $\II^a$ is denoted by
    \[ \PP^a(\cdot) = \PP(\cdot \given \mcond^a), \] where $\mcond^a$ is the assignment to $\MCond^a$ in keeping with our conventions.
\end{definition}

The outputs of applying a $k$-MixProd oracle to  $\PP^a(\II^a)$ are a matrix $\vec{M}^a \in [0,1]^{\Card{\II^a}\times k}$ and a vector of mixture weights, $\pi^a \in [0,1]^k$ (satisfying $\sum_u \pi^a_u = 1$) given by
\begin{align}
    M^a_{i,u} & \coloneqq \PP^a(X^a_i=1\given U_a=u) = \PP(X_i=1\given U^a = u, \MCond^a), & \quad \forall X_i \in \II^a, u\in [k] \\\pi^a_{u} &\coloneqq \PP^a(U^a = u) = \PP(U_a=u\given \MCond^a),&\quad \forall j\in [k]
\end{align} where $U^a$ is the mixture source distributed over $[k]$ according to $\PP(U\given \MCond^a)$.  Note that because mixtures are invariant to permutations of mixture component labels, we cannot guarantee correspondence between the labels for the source variables from different runs.  Hence the labels of $U^a$ to an unknown permutation of the labels in $U$.  Alignment of these labels is handled in Section~\ref{sec: aligning source labels}.

\section{Combining Runs}\label{sec:combining}
A single run of the $k$-MixProd oracle will not contain sufficient information to learn the parameters of the $k$-MixBND problem. Instead we must synthesize information across \emph{multiple} runs.
\subsection{Aligning source labels across different runs} \label{sec: aligning source labels}
Each run of the $k$-MixProd oracle will return $\PP^a(V\given U_a = u)$ for some arbitrary permutation $U_a$ of the variable. We need to align all of the outputs to the same permutation of the source, $U$.  If the runs overlap on at least one variable with the same mixture probabilities, we can use that ``alignment variable'' to identify which source corresponds to which set of parameters.  In our setup, we will guarantee these alignment variables exist by ensuring that runs have shared vertices in their independent sets whose Markov boundaries have identical assignments.

\begin{definition}
$X\in \vec{V}$ is \textbf{separated} if for all $u_i \neq u_j  \in [k]$, $\PP_{u_i}(x) \neq \PP_{u_j}(x)$. 
\end{definition}

\begin{definition}[Aligned runs]\label{def:two_mixtures_aligned}
    A pair of runs $a,b$ over independent sets $\vec{I}^{a}, \vec{I}^{b}$ is \textbf{alignable} if there exists a separated $X \in \vec{I}^{(a)} \cap \vec{I}^{(b)}$ such that $\PP^a_u(V^{(ab)}) = \PP^b_u(V^{(ab)})$ for all $u\in [k]$. We'll call any such random variable $X$ an \textbf{alignment variable}, and use $\AV(a,b)$ to denote the set of all alignment variables for $\PP^{a}$ and $\PP^{b}$.  We sometimes say $a$ and $b$ are ``aligned at'' $X$.
\end{definition}

\begin{definition}[Alignment spanning tree]
    We say a set of $\ell$ runs is alignable if there exists an undirected spanning tree over the graph with vertices $a_1, \ldots a_\ell$ and an edge $\{a_i, a_j\}$ whenever $\AV(a_i, a_j) \neq \emptyset$. We call this the \textbf{alignment spanning tree}.
\end{definition}

The alignment step will take the output from alignable runs and permute the mixture labels until the parameters match along each alignment variable.  Pseudocode for this procedure is given in Algorithm~\ref{alg:align}.

\subsection{Bayesian unzipping: recovering parameters per source}
Recall that our algorithm conditions on Markov boundaries to induce independent variables.  Hence,  after aligning the sources in runs of the $k$-MixProd oracle we will have access to $\PP_u(Y \given \Mb(Y))$ for each $Y \in V$.  Our goal is to obtain $\PP_u(\vec{V})$,  which is described by the parameters $\PP_u(Y\given \Pa(Y))$ for each $Y \in \vec{V}$.  Note that
\begin{equation}\label{eq:mbdecomp}
    \PP_u(y^1 \given \mb(Y)) = \frac{\PP_u(y^1 , \mb(Y))}{\PP_u(y^1 , \mb(Y)) + \PP_u(y^0 , \mb(Y))}.
\end{equation}
The terms in this fraction are all of the same form and can be factored according to the DAG into
\begin{equation*}
    \PP_u(y, \mb^a(Y)) = \PP_u(\ttop(Y)) \PP_u(y \given \pa^a(Y)) \underbrace{\prod_{V \in \Ch(Y)}\PP_u(v^a \given f^{a}(\Pa(V) \setminus \{Y\}),  y)}_{\PP_u(\ch^a(Y) \given \ttop^a(Y), y)}.
\end{equation*} See Figure~\ref{fig:MBdecomp} for a concrete example of this decomposition.
After substituting this factorization into Equation~\eqref{eq:mbdecomp} we see that $\PP_u(\ttop(Y))$ appears in both the numerator and denominator because it is independent of the assignment to $Y$.  Simplification leaves only the following terms:
\begin{enumerate}[nosep]
    \item $\PP_u(y^0 \given \pa^a(Y))$ and $\PP_u(y^1 \given \pa^a(Y))$, which must sum to 1.
    \item  $\PP_u(\ch^a(Y) \given \ttop^a(Y), y^0)$ and  $\PP_u(\ch^a(Y) \given \ttop^a(Y), y^1)$ which are both the product of the desired parameters of variables later in the topological ordering.  We can ensure we have access to these terms by solving for the parameters of $V \in \Vars$ in a reverse-topological ordering.\footnote{We will want to ensure that we only need to unzip parameters from vertices of a bounded depth, which bounds the iterations of this step.  Details on how this is done appear in Section~\ref{sec:bottom vertices}.}
\end{enumerate}
We can substitute $1-\PP_u(y^1\given \pa^a(Y))$ for $\PP_u(y^0\given \pa^a(Y))$ in the expanded version of Equation~\eqref{eq:mbdecomp} to obtain a single equation with only $\PP_u(y^1\given \pa^a(Y))$ as an unknown, which we can then solve.  The pseudocode for this process is given in Algorithm~\ref{alg:bayesian-unzipping}.
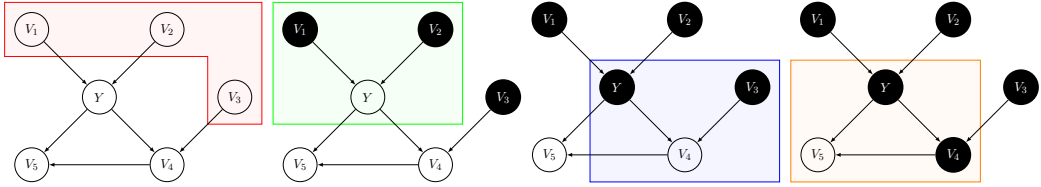
\begin{figure}[h]
    \setlength{\belowcaptionskip}{-10pt}
    \centering
    \scalebox{.45}{
        \begin {tikzpicture}[baseline={([yshift=-.5ex]current bounding box.center)}, -latex ,auto ,node distance =2 cm and 2 cm ,on grid , semithick, state/.style ={ circle, draw, minimum width =1 cm}]
        \filldraw[color = red, fill=red!4](-2.8, 2.8) -- (4.8, 2.8) -- (4.8, -.8) -- (3.2, -.8) -- (3.2, 1.2) -- (-2.8, 1.2) -- cycle;
        \node[state] (v) {$Y$};
        \node[state] (p1)[above left = of v] {$V_1$};
        \node[state] (p2)[above right = of v] {$V_2$};
        \node[state] (c1)[below left = of v] {$V_5$};
        \node[state] (c2)[below right = of v] {$V_4$};
        \node[state] (x)[above right = of c2] {$V_3$};
        \path (p1) edge (v) (p2) edge (v) (v) edge (c1) (v) edge (c2) (x) edge (c2) (c2) edge (c1);
        \end{tikzpicture}}
    \scalebox{.45}{
        \begin {tikzpicture}[baseline={([yshift=-.5ex]current bounding box.center)}, -latex ,auto ,node distance =2 cm and 2 cm ,on grid , semithick, state/.style ={ circle, draw, minimum width =1 cm}, cstate/.style ={ circle, draw, minimum width =1 cm, fill=black, text=white, ultra thick}]
        \filldraw[color = green, fill=green!4](-2.8, 2.8) rectangle (2.8, -.8);
        \node[state] (v) {$Y$};
        \node[cstate] (p1)[above left = of v] {$V_1$};
        \node[cstate] (p2)[above right = of v] {$V_2$};
        \node[state] (c1)[below left = of v] {$V_5$};
        \node[state] (c2)[below right = of v] {$V_4$};
        \node[cstate] (x)[above right = of c2] {$V_3$};
        \path (p1) edge (v) (p2) edge (v) (v) edge (c1) (v) edge (c2) (x) edge (c2) (c2) edge (c1);
        \end{tikzpicture}}
    \scalebox{.45}{
        \begin {tikzpicture}[baseline={([yshift=-.5ex]current bounding box.center)}, -latex ,auto ,node distance =2 cm and 2 cm ,on grid , semithick, state/.style ={ circle, draw, minimum width =1 cm}, cstate/.style ={ circle, draw, minimum width =1 cm, fill=black, text=white, ultra thick}]
        \filldraw[color = blue, fill=blue!4](-.8, .8) rectangle (4.8, -2.8);
        \node[cstate] (v) {$Y$};
        \node[cstate] (p1)[above left = of v] {$V_1$};
        \node[cstate] (p2)[above right = of v] {$V_2$};
        \node[state] (c1)[below left = of v] {$V_5$};
        \node[state] (c2)[below right = of v] {$V_4$};
        \node[cstate] (x)[above right = of c2] {$V_3$};
        \path (p1) edge (v) (p2) edge (v) (v) edge (c1) (v) edge (c2) (x) edge (c2) (c2) edge (c1);
        \end{tikzpicture}}
    \scalebox{.45}{
        \begin {tikzpicture}[baseline={([yshift=-.5ex]current bounding box.center)}, -latex ,auto ,node distance =2 cm and 2 cm ,on grid , semithick, state/.style ={ circle, draw, minimum width =1 cm}, cstate/.style ={ circle, draw, minimum width =1 cm, fill=black, text=white, ultra thick}]
        \filldraw[color = orange, fill=orange!4](-2.8, .8) rectangle (2.8, -2.8);
        \node[cstate] (v) {$Y$};
        \node[cstate] (p1)[above left = of v] {$V_1$};
        \node[cstate] (p2)[above right = of v] {$V_2$};
        \node[state] (c1)[below left = of v] {$V_5$};
        \node[cstate] (c2)[below right = of v] {$V_4$};
        \node[cstate] (x)[above right = of c2] {$V_3$};
        \path (p1) edge (v) (p2) edge (v) (v) edge (c1) (v) edge (c2) (x) edge (c2) (c2) edge (c1);
        \end{tikzpicture}}
    \caption{We can decompose $\PP_u(v_1,v_2,  v_3, y,  v_4, v_5) = \color{red}\PP_u(v_1, v_2, v_3)\color{black} \color{green!60!black}\PP_u(y \given v_1, v_2)\color{black}\color{blue}\PP_u(v_4 \given y, v_3)\color{black}  \color{orange!70!black} \PP_u(v_5 \given y, v_4) \color{black}$.  $U$ and any other variables in the graph are omitted for clarity.
    }\label{fig:MBdecomp}
\end{figure}

\subsubsection{Recovering the distribution on sources} \label{sec:recoversourceprobability} Now consider some arbitrary run $a$ with conditioning $\mcond^a$. Since
$\PP_u(\Vars) = \prod_{V\in \Vars} \PP_u(V\given \Pa(V))$, knowing $\PP_u(V\given \Pa(V))$ grants us full access to the within-source probability distribution $\PP_u(\Vars)$ after Bayesian unzipping.  From this we can we obtain $\PP_u(\mcond^a) = \PP(\mcond^a\given u)$. The $k$-MixProd oracle will also return $\PP^a(U) = \PP(U\given \mcond^a)$ when run on $a$ (after source alignment).  Finally,  $\PP(\mcond^a)$ is directly observable. Combining these terms in Bayes' rule lets us compute the distribution on $U$ (under the assumption of positivity),
\begin{equation*}
    \PP(u) = \frac{\PP(u \given \mcond^a) \PP(\mcond^a)}{\PP(\mcond^a \given u)}.
\end{equation*}

\subsection{Outline of the combination process}
Combining a set of runs $\calA$ has four steps.
\begin{enumerate}[ nosep]
    \item Use a $k$-MixProd oracle on $\PP(\II^a\given \MCond^a)$ for each run $a\in \calA$ to compute $P(V\given \Mb(V), U_a = u)$ for all variables $V\in \Vars$.
    \item Align the parameters obtained from the previous step to ensure that $U$ means the same thing across different runs, giving $\PP_u(V\given \Mb(Y))$. 
    \item Recover $\PP_u(V\given \Pa(V))$ for each vertex $V\in \Vars$ via Bayesian unzipping.
    \item Compute $\PP(U)$ by applying Bayes' law.
\end{enumerate}
The full procedure appears as Algorithm~\ref{alg:k-bnd}.

\section{Collections of runs}\label{sec:good_collection}
With the main concepts of source alignment and Bayesian unzipping now defined,  our algorithm will primarily consist of finding a good collection of these runs so that these subroutines can be successfully applied to recover the $k$-MixBND mixture.

\begin{observation}
    Two runs $a,b$ are aligned at $X\in \Vars$ if and only if
    \begin{enumerate}[nosep]
        \item $X \in \II^a \cap \II^b$,
        \item $\mb^a(X) = \mb^b(X)$, i.e, $f^a(\Mb(X)) = f^b(\Mb(X))$, and
        \item $X$ is separated given $\mb^a(X)$ (equivalently, given $\mb^b(X)$).
    \end{enumerate}
\end{observation}

\begin{definition} A collection of runs $\calA$ \textbf{covers} $X \in \Vars$ if for every assignment $\pa(X)$ to $\Pa(X)$ there exists a run $a\in \calA$ with $X \in \II^a$ and $\pa(X) = \pa^a(X)$.
\end{definition}

\begin{definition}[A good collection of runs]
\label{def: good collection}
    A collection of well-formed runs $\calA$ is \emph{good} if it is (i) alignable via an alignment spanning tree, (ii) every run is $N_{\text{mp}}$-independent, and (iii) the collection covers every vertex in $\Vars$.
\end{definition}
The following is our main result on good collections of runs:
\begin{theorem}
 Given a graph with max degree $\Delta$ satisfying $n \geq N_{\text{mp}}\times O(\Delta^4)$, we can find a set of centers $\Centers = \Set{X_1,\dotsc, X_{N_{\text{mp}}}} \subseteq \Vars$ of size $N_{\text{mp}}$ and depth at most $3N_{\text{mp}}$, such that Algorithm~\ref{alg:good runs} succeeds in finding a good collection of runs $\calA$ of size $O(2^{\Delta^2} n)$. \label{thm:main result on good collection}
\end{theorem}
While any good collection of runs will suffice for our algorithm, Theorem~\ref{thm:main result on good collection} represents conditions under which we can provably obtain such a collection of runs.  In Appendix~\ref{apx:paths}, we give a good collection of runs for mixture of paths which is more efficient than the one found by Algorithm~\ref{alg:good runs}.

\subsection{A generic good collection of runs}
To prove Theorem~\ref{thm:main result on good collection}, we will sketch the collection of good runs, leaving some details to Appendix~\ref{apx: good collections}.  To ensure alignment is possible,  we will construct a set of \emph{central runs}, $\calA_C$ which we can align to each other and which all other runs will be alignable to.
\begin{definition}[Centers, Central Runs] A set of vertices $\Centers = \Set{X_1,\dotsc, X_{N_{\text{mp}}}}\subseteq \Vars$ will be called \textbf{centers} if the Markov boundaries of the vertices in $\Centers$ are disjoint. Given a set of centers $\Centers$, a run $a$ is called a \textbf{central run} if $\II^a = \Centers$.
\end{definition}
To build these central runs, we will start with a set of $N_{\text{mp}}$ vertices $\Centers=\Set{X_1, X_2, \dotsc X_{N_{\text{mp}}}}$ with \emph{disjoint} Markov boundaries and a maximum depzth of $3 N_{\text{mb}}$, whose existence is implied by our degree bounds (see Appendix~\ref{apx: good collections}). An example of four such vertices is given in Figure~\ref{fig:disjointmb}.

First,  we fix a run $a_0$ with $\II^{a_0} = \Centers$ and $\mb^{a_0}(\Centers)$ being chosen arbitrarily where $\Mb(\Centers) \coloneqq \cup_{X_i\in \Centers}\Mb(X_i)$.  We will refer to this assignment $\mb^{a_0}(\Centers)$ as the \emph{default} assignment.  Each run in $a\in \calA_C$ will have the same independent set $\II^a = \Centers$ and will agree with $a_0$ on the assignment to all of the conditioning set other than the Markov  boundary of some $X_i \in \Centers$, i.e.
\begin{equation}
f^{a}(V) = f^{a_0}(V) \;\; \forall V \in \Mb(\vec{X}) \setminus \Mb(X_i).
\end{equation}
The central runs will span over all assignments $\mb(X_i)$ to $\Mb(X_i)$ for each $X_i \in \Centers$.  We'll write each such run as $a_0[\mb(X_i)]$.
\begin{definition}
    $\calA_C \coloneqq \Set{a_0} \cup \Set{a_0[ \mb(X_i)]:i\in [3k-3], \mb(X_i) \in \Set{0,1}^{\Mb(X_i)}}$.
\end{definition}
See Figure~\ref{fig:centralruns} for an example of a set of central runs and a visualization of how they are alignable.
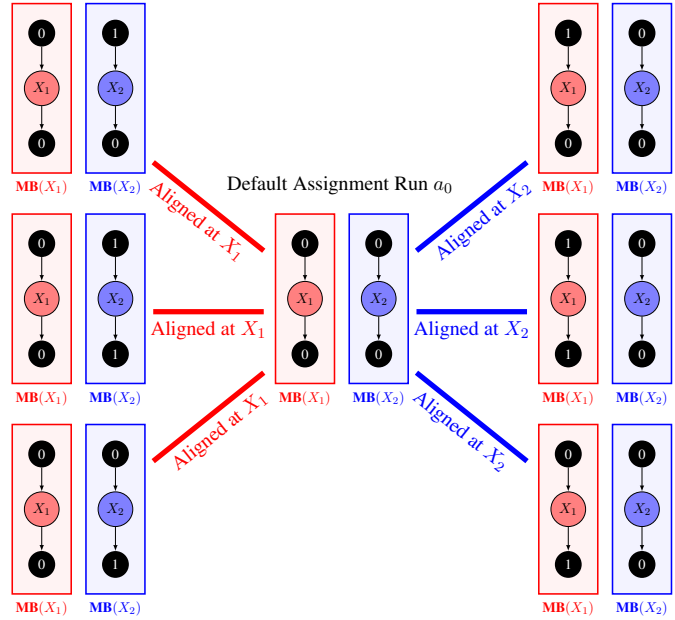
\begin{figure}[h]
\setlength{\belowcaptionskip}{-10pt}
\newcommand{\runex}[8]{
    \scalebox{.7}{
    \begin {tikzpicture}[-latex ,auto ,node distance =2 cm and 2 cm ,on grid , semithick, state/.style ={ circle, draw, minimum width =.5 cm},  cstate/.style ={ circle, draw, minimum width =.5 cm, fill=black, text=white, very thick}]
            \draw[color=red, very thick, fill=red, fill opacity = .05](-1.8,  2.3) rectangle (-.2, -2.3);
            \node[color=red] (llab) at (-1, -2.7) {$\Mb(X_1)$};
             \draw[color=blue, very thick, fill=blue, fill opacity = .05](.2, 2.3) rectangle (1.8, -2.3);
            \node[color=blue] (llab) at (1, -2.7) {$\Mb(X_2)$};
            \node[cstate] (t1) at (-1, 1.5) {#1};
            \node[state, fill=red!50!white] (x1) at (-1, 0) {$X_1$};
            \node[cstate] (b1) at (-1, -1.5) {#2};
            \node[cstate] (t2) at (1, 1.5) {#3};
            \node[state, fill=blue!50!white] (x2) at (1, 0) {$X_2$};
            \node[cstate] (b2) at (1, -1.5) {#4};
            \path (t1) edge (x1) (x1) edge (b1) (t2) edge (x2) (x2) edge (b2);
        \end{tikzpicture}}}
\centering
\scalebox{.7}{\begin {tikzpicture}[auto ,node distance =4 cm and 5 cm ,on grid , semithick, state/.style ={ circle, draw, minimum width =.5 cm}]
            \node (astar) {\runex{0}{0}{0}{0}{blue}{blue}{blue}{blue}};
            \node (tl) [above right =of astar] {\runex{1}{0}{0}{0}{red}{blue}{blue}{blue}};
            \node (tc) [right =of astar] {\runex{1}{1}{0}{0}{red}{red}{blue}{blue}};
            \node (tr) [below right =of astar] {\runex{0}{1}{0}{0}{blue}{red}{blue}{blue}};
            \node (bl) [above left =of astar] {\runex{0}{0}{1}{0}{blue}{blue}{red}{blue}};
            \node (bc) [left =of astar] {\runex{0}{0}{1}{1}{blue}{blue}{red}{red}};
            \node (br) [below left =of astar] {\runex{0}{0}{0}{1}{blue}{blue}{blue}{red}};
            \path[line width = 1mm, color=blue] (tl) edge node [sloped, anchor=center, below] {Aligned at $X_2$} (astar) 
            (tc) edge node [sloped, anchor=center, below] {Aligned at $X_2$} (astar) 
            (tr) edge node [sloped, anchor=center, below] {Aligned at $X_2$}(astar);
            \path[line width = 1mm, color=red] (bl) edge node [sloped, anchor=center, below] {Aligned at $X_1$} (astar)
            (bc) edge node [sloped, anchor=center, below] {Aligned at $X_1$} (astar) 
            (br) edge node [sloped, anchor=center, below] {Aligned at $X_1$} (astar);
            \node[] (astarlable) at (0, 2.4){Default Assignment Run $a_0$};
\end{tikzpicture}}
\caption{An alignment spanning tree of the default assignment $a_0$ ($\MCond^{a_0}$ arbitrarily assigns all Markov boundaries to $0$) and six other central runs.  The runs on the left cover all possible assignments to $\Mb(X_2) \in \Set{(0,0), (0,1), (1, 0)}$, while maintaining the default assignment to $\Mb(X_1)$ to allow alignment with $a_0$.  The right runs similarly cover all possible assignments to $\Mb(X_1)$, aligned at $X_2$.}
\label{fig:centralruns}
\end{figure}

The central runs provide a backbone for easily guaranteeing alignment. The runs in $\calA_Y$ made up of the following two types of perturbations to the independent set (with $\Cond^{a}$ always defined as the union of the Markov boundaries of the independent set, as in Definition~\ref{def:run}): \begin{enumerate}[nosep]
    \item For each $Y \in \Mb(X_i)$ for some $X_i \in \Centers$,  we exclude $X_i$ from the independent set to form $\II^{a_Y} = \Centers \cup \{Y\} \setminus \{X_i\}$.\footnote{This is a well-formed run since $Y\in \Mb(X_i) \implies X_j \notin \Mb(Y)$ for any $j\neq i$ by Observation~\ref{obs:mb mb}.}
    \item For each $Y\notin \Mb(\Centers) \cap \Centers$, $\II^a = \Centers \cup \{Y\}$.
\end{enumerate}
For either independence set we will form $2^{\abs{\Pa(Y)}}$ runs each associated with a single assignment to $\pa^a(Y)$, with the remaining variables in $\MCond^a \cap \MCond^{a_0}$ conditioned on their defaults given by $f^{a_0}$.  Any other assignments to variables in $\MCond^a$ can be chosen arbitrarily. The pseudocode for this construction is in Appendix~\ref{apx: good collections},  as well as justification for why our degree bounds imply that $N_{\text{mp}}$ disjoint Markov boundaries can be found.

\subsection{Limiting the depth of unzipping}\label{sec:bottom vertices}
As currently given, our algorithm may require Bayesian unzipping parameters up to the depth of the graph.  We can bound accumulated errors from this process by limiting the depth of the vertices that need to be unzipped. 

Recall that the goal of the conditioning set of each run is to $d$-separate each of the vertices in the independent set.  For a topological ordering on the independence set, notice that we need not condition on the descendants of the deepest vertices in order to $d$-separate them from the others.  Conveniently,  avoiding conditioning on these vertices descendants leaves the output of the $k$-MixProd oracle in the desired form. We call these deepest vertices ``bottom'' vertices.

\begin{definition}[Bottom vertices in a run] \label{def:bottom vertex} Given vertices $\II$, we'll define the set of bottom vertices of the run to be the subset $\Bottom(\II) \subseteq \II$ of vertices with maximal depth among the vertices in $\II$. That is $d(B) = \max_{I\in \II} d(I)$ for all $B \in \Bottom(\II)$. 
\end{definition}
We can now update the conditioning sets for our definition of runs:
\begin{equation}
\MCond^a \coloneqq \bigcup_{I \in \II^a \setminus \Bottom(\II^a)} \Mb(I) \bigcup_{B\in \Bottom(\II^a)} \Pa(B)
\end{equation}
We append two additional requirements for a good collection of runs (Definition~\ref{def: good collection}).
\begin{itemize}[nosep]
	\item no vertex appears both as a bottom vertex and a non-bottom vertex, and
     \item every non-bottom vertex has depth at most $3N_{\text{mp}}$.
\end{itemize}
Note that because we only need independent sets of size $N_{\text{mp}}$, it is trivial to limit the depth of our non-bottom vertices to  $3N_{\text{mp}}$.

\section{Conclusion}
We have developed the first algorithm for identifying the parameters of $k$-MixBND mixtures.  This algorithm allows us to access the probability distributions within each source - equivalent to the probability distribution conditioned on a universal and unobserved confounder.  With access to this conditional distribution, the confounding of $U$ can be adjusted for, opening up the opportunity for causal inference despite universal confounding.

The algorithm presented here is intended as an identifiability result.  Instead, we hope this algorithm and framework can serve as a springboard for understanding how solutions to the $k$-MixProd and $k$-MixBND problems are intimately related.

The alignment process is highly nontrivial and a likely reason why papers such as \citet{anandkumar2012learning} made crude assumptions  (such as a conveniently independent variable) to simplify this process. The formal development of a notion of a run, while tedious, will allow for further improvements to give better ``good sets of runs.'' Graph-specific sets of runs can be optimized further, as demonstrated in Appendix~\ref{apx:paths}. 
\newpage
\bibliography{refs}

\begin{thebibliography}{46}
\providecommand{\natexlab}[1]{#1}
\providecommand{\url}[1]{\texttt{#1}}
\expandafter\ifx\csname urlstyle\endcsname\relax
  \providecommand{\doi}[1]{doi: #1}\else
  \providecommand{\doi}{doi: \begingroup \urlstyle{rm}\Url}\fi

\bibitem[Anandkumar et~al.(2010)Anandkumar, Tan, and
  Willsky]{anandkumar2010high}
A.~Anandkumar, V.~Tan, and A.~Willsky.
\newblock High dimensional structure learning of ising models on sparse random
  graphs.
\newblock 2010.

\bibitem[Anandkumar et~al.(2012{\natexlab{a}})Anandkumar, Hsu, Huang, and
  Kakade]{anandkumar2012learning}
A.~Anandkumar, D.~Hsu, F.~Huang, and S.~M. Kakade.
\newblock Learning high-dimensional mixtures of graphical models.
\newblock \emph{arXiv preprint arXiv:1203.0697}, 2012{\natexlab{a}}.

\bibitem[Anandkumar et~al.(2012{\natexlab{b}})Anandkumar, Hsu, and
  Kakade]{anandkumar2012method}
A.~Anandkumar, D.~J. Hsu, and S.~M. Kakade.
\newblock A method of moments for mixture models and hidden {M}arkov models.
\newblock In \emph{Proc. 25th Ann. Conf. on Learning Theory - {COLT}},
  volume~23 of \emph{{JMLR} Proceedings}, pages 33.1--33.34,
  2012{\natexlab{b}}.
\newblock URL
  \url{http://proceedings.mlr.press/v23/anandkumar12/anandkumar12.pdf}.

\bibitem[Batu et~al.(2004)Batu, Guha, and Kannan]{BGK04}
T.~Batu, S.~Guha, and S.~Kannan.
\newblock Inferring mixtures of {M}arkov chains.
\newblock In \emph{Proc. 17th Conf. on Learning Theory}, pages 186--199, 2004.
\newblock \doi{10.1007/978-3-540-27819-1_13}.

\bibitem[Chaudhuri and Rao(2008)]{CR08}
K.~Chaudhuri and S.~Rao.
\newblock Learning mixtures of product distributions using correlations and
  independence.
\newblock In \emph{Proc. 21st Ann. Conf. on Learning Theory - {COLT}}, pages
  9--20. Omnipress, 2008.
\newblock URL \url{http://colt2008.cs.helsinki.fi/papers/7-Chaudhuri.pdf}.

\bibitem[Chen and Moitra(2019)]{ChenMoitra19}
S.~Chen and A.~Moitra.
\newblock Beyond the low-degree algorithm: mixtures of subcubes and their
  applications.
\newblock In \emph{Proc. 51st Ann. {ACM} Symp. on Theory of Computing}, pages
  869--880, 2019.
\newblock \doi{10.1145/3313276.3316375}.

\bibitem[Cryan et~al.(2001)Cryan, Goldberg, and Goldberg]{CGG01}
M.~Cryan, L.~Goldberg, and P.~Goldberg.
\newblock Evolutionary trees can be learned in polynomial time in the two state
  general {M}arkov model.
\newblock \emph{SIAM J. Comput.}, 31\penalty0 (2):\penalty0 375--397, 2001.
\newblock \doi{10.1137/S0097539798342496}.

\bibitem[E.~S.~Allman(2009)]{allman2009identifiability}
J.~A.~Rhodes E.~S.~Allman, C.~Matias.
\newblock Identifiability of parameters in latent structure models with many
  observed variables.
\newblock \emph{Ann. Statist.}, 37\penalty0 (6A):\penalty0 3099--3132, 2009.
\newblock \doi{10.1214/09-AOS689}.

\bibitem[Everitt and Hand(1981)]{Everitt1981}
B.~S. Everitt and D.~J. Hand.
\newblock Mixtures of discrete distributions.
\newblock In \emph{Finite Mixture Distributions}, pages 89--105. Springer
  Netherlands, Dordrecht, 1981.

\bibitem[Feldman et~al.(2008)Feldman, O'Donnell, and Servedio]{FOS08}
J.~Feldman, R.~O'Donnell, and R.~A. Servedio.
\newblock Learning mixtures of product distributions over discrete domains.
\newblock \emph{SIAM J. Comput.}, 37\penalty0 (5):\penalty0 1536--1564, 2008.
\newblock \doi{10.1137/060670705}.

\bibitem[Freund and Mansour(1999)]{FM99}
Y.~Freund and Y.~Mansour.
\newblock Estimating a mixture of two product distributions.
\newblock In \emph{Proc. 12th Ann. Conf. on Computational Learning Theory},
  pages 53--62, July 1999.
\newblock \doi{10.1145/307400.307412}.

\bibitem[Glymour et~al.(2019{\natexlab{a}})Glymour, Zhang, and
  Spirtes]{10.3389/fgene.2019.00524}
C.~Glymour, K.~Zhang, and P.~Spirtes.
\newblock Review of causal discovery methods based on graphical models.
\newblock \emph{Frontiers in Genetics}, 10, 2019{\natexlab{a}}.
\newblock ISSN 1664-8021.
\newblock \doi{10.3389/fgene.2019.00524}.
\newblock URL
  \url{https://www.frontiersin.org/article/10.3389/fgene.2019.00524}.

\bibitem[Glymour et~al.(2019{\natexlab{b}})Glymour, Zhang, and
  Spirtes]{causationreview}
C.~Glymour, K.~Zhang, and P.~Spirtes.
\newblock Review of causal discovery methods based on graphical models.
\newblock \emph{Frontiers in Genetics}, 10:\penalty0 524, 2019{\natexlab{b}}.
\newblock \doi{10.3389/fgene.2019.00524}.

\bibitem[Gordon and Schulman(2022)]{gordon2022hadamard}
S.~L. Gordon and L.~J. Schulman.
\newblock Hadamard extensions and the identification of mixtures of product
  distributions.
\newblock \emph{IEEE Transactions on Information Theory}, 2022.
\newblock to appear.

\bibitem[Gordon et~al.(2021)Gordon, Mazaheri, Rabani, and
  Schulman]{gordon2021source}
S.~L. Gordon, B.~Mazaheri, Y.~Rabani, and L.~J. Schulman.
\newblock Source identification for mixtures of product distributions.
\newblock In \emph{Proc. 34th Ann. Conf. on Learning Theory - {COLT}}, volume
  134 of \emph{Proc. Machine Learning Research}, pages 2193--2216. {PMLR},
  2021.
\newblock URL \url{http://proceedings.mlr.press/v134/gordon21a.html}.

\bibitem[Gupta et~al.(2016)Gupta, Kumar, and Vassilvitskii]{GKV16}
R.~Gupta, R.~Kumar, and S.~Vassilvitskii.
\newblock On mixtures of {M}arkov chains.
\newblock In \emph{Advances in Neural Information Processing Systems},
  volume~29, 2016.
\newblock URL
  \url{https://proceedings.neurips.cc/paper/2016/file/8b5700012be65c9da25f49408d959ca0-Paper.pdf}.

\bibitem[Heckerman(2018)]{heckerman2018accounting}
D.~Heckerman.
\newblock Accounting for hidden common causes when inferring cause and effect
  from observational data.
\newblock (NIPS 2017 causal inference workshop), 2018.
\newblock URL \url{https://arxiv.org/abs/1801.00727}.

\bibitem[Hsu et~al.(2012)Hsu, Kakade, and Zhang]{hsuKZ12}
D.~Hsu, S.~M. Kakade, and T.~Zhang.
\newblock A spectral algorithm for learning hidden {M}arkov models.
\newblock \emph{J. Comput. Syst. Sci.}, 78\penalty0 (5):\penalty0 1460–1480,
  September 2012.
\newblock \doi{10.1016/j.jcss.2011.12.025}.

\bibitem[Huang and Valtorta(2008)]{huangV08}
Y.~Huang and M.~Valtorta.
\newblock On the completeness of an identifiability algorithm for
  semi-{Markovian} models.
\newblock \emph{Ann. Math. Artif. Intell.}, 54\penalty0 (4):\penalty0 363--408,
  December 2008.
\newblock \doi{/10.1007/s10472-008-9101-x}.

\bibitem[Kearns et~al.(1994)Kearns, Mansour, Ron, Rubinfeld, Schapire, and
  Sellie]{KMRRSS94}
M.~Kearns, Y.~Mansour, D.~Ron, R.~Rubinfeld, R.~Schapire, and L.~Sellie.
\newblock On the learnability of discrete distributions.
\newblock In \emph{Proc. 26th Ann. {ACM} Symp. on Theory of Computing}, pages
  273--282, 1994.
\newblock \doi{10.1145/195058.195155}.

\bibitem[Kivva et~al.(2021)Kivva, Rajendran, Ravikumar, and
  Aragam]{kivva2021learning}
Bohdan Kivva, Goutham Rajendran, Pradeep~Kumar Ravikumar, and Bryon Aragam.
\newblock Learning latent causal graphs via mixture oracles.
\newblock In A.~Beygelzimer, Y.~Dauphin, P.~Liang, and J.~Wortman Vaughan,
  editors, \emph{Advances in Neural Information Processing Systems}, 2021.
\newblock URL \url{https://openreview.net/forum?id=f9mSLa07Ncc}.

\bibitem[Koller and Friedman(2009)]{KF09}
D.~Koller and N.~Friedman.
\newblock \emph{Probabilistic Graphical Models: Principles and Techniques}.
\newblock MIT Press, 2009.

\bibitem[Kruskal(1976)]{kruskal1976more}
Joseph~B Kruskal.
\newblock More factors than subjects, tests and treatments: an indeterminacy
  theorem for canonical decomposition and individual differences scaling.
\newblock \emph{Psychometrika}, 41\penalty0 (3):\penalty0 281--293, 1976.

\bibitem[Kruskal(1977)]{kruskal1977three}
Joseph~B Kruskal.
\newblock Three-way arrays: rank and uniqueness of trilinear decompositions,
  with application to arithmetic complexity and statistics.
\newblock \emph{Linear algebra and its applications}, 18\penalty0 (2):\penalty0
  95--138, 1977.

\bibitem[Kumar and Sinha(2021)]{kumarS21}
A.~Kumar and G.~Sinha.
\newblock Disentangling mixtures of unknown causal interventions.
\newblock In C.~de~Campos and M.~H. Maathuis, editors, \emph{Proc.
  Thirty-Seventh Conference on Uncertainty in Artificial Intelligence}, volume
  161 of \emph{Proc. Machine Learning Research}, pages 2093--2102. PMLR, 27--30
  Jul 2021.
\newblock URL \url{https://proceedings.mlr.press/v161/kumar21a.html}.

\bibitem[Kuroki and Pearl(2014)]{kuroki2014measurement}
M.~Kuroki and J.~Pearl.
\newblock Measurement bias and effect restoration in causal inference.
\newblock \emph{Biometrika}, 101\penalty0 (2):\penalty0 423--437, 2014.
\newblock \doi{10.1093/biomet/ast066}.

\bibitem[Lindsay(1995)]{lindsay1995mixture}
B.~G. Lindsay.
\newblock \emph{Mixture models: theory, geometry and applications}.
\newblock 1995.

\bibitem[Mazaheri et~al.(2023)Mazaheri, Mastakouri, Janzing, and
  Hardt]{mazaheri2023causal}
Bijan Mazaheri, Atalanti Mastakouri, Dominik Janzing, and Mila Hardt.
\newblock Causal information splitting: Engineering proxy features for
  robustness to distribution shifts, 2023.

\bibitem[McLachlan et~al.(2019)McLachlan, Lee, and Rathnayake]{McLachlanLR19}
G.~J. McLachlan, S.~X. Lee, and S.~I. Rathnayake.
\newblock Finite mixture models.
\newblock \emph{Annual Review of Statistics and Its Application}, 6\penalty0
  (1):\penalty0 355--378, 2019.
\newblock \doi{10.1146/annurev-statistics-031017-100325}.

\bibitem[Miao et~al.(2018)Miao, Geng, and Tchetgen]{miao2018identifying}
W.~Miao, Z.~Geng, and E.~T. Tchetgen.
\newblock Identifying causal effects with proxy variables of an unmeasured
  confounder.
\newblock \emph{Biometrika}, 105\penalty0 (4):\penalty0 987--993, 2018.
\newblock \doi{10.1093/biomet/asy038}.

\bibitem[Newcomb(1886)]{Newcomb86}
S.~Newcomb.
\newblock A generalized theory of the combination of observations so as to
  obtain the best result.
\newblock \emph{American Journal of Mathematics}, 8\penalty0 (4):\penalty0
  343--366, 1886.

\bibitem[Ogburn et~al.(2019)Ogburn, Shpitser, and Tchetgen]{ogburn2019comment}
Elizabeth~L Ogburn, Ilya Shpitser, and Eric J~Tchetgen Tchetgen.
\newblock Comment on “blessings of multiple causes”.
\newblock \emph{Journal of the American Statistical Association}, 114\penalty0
  (528):\penalty0 1611--1615, 2019.

\bibitem[Pearl(1985)]{pearl1985bayesian}
J.~Pearl.
\newblock Bayesian networks: {A} model of self-activated memory for evidential
  reasoning.
\newblock Technical Report CSD-850021, R-43, UCLA Computer Science Department,
  June 1985.

\bibitem[Pearl(2009)]{Pea09}
J.~Pearl.
\newblock \emph{Causality}.
\newblock Cambridge, 2nd edition, 2009.

\bibitem[Pearl(2014)]{pearl2014probabilistic}
J.~Pearl.
\newblock \emph{Probabilistic reasoning in intelligent systems: networks of
  plausible inference}.
\newblock Elsevier, 2014.

\bibitem[Pearson(1894)]{Pearson94}
K.~Pearson.
\newblock {C}ontributions to the mathematical theory of evolution {III}.
\newblock \emph{Philosophical Transactions of the Royal Society of London
  (A.)}, 185:\penalty0 71--110, 1894.

\bibitem[Peters et~al.(2017)Peters, Janzing, and Sch\"olkopf]{PetersJS17}
J.~Peters, D.~Janzing, and B.~Sch\"olkopf.
\newblock \emph{Elements of Causal Inference}.
\newblock MIT Press, 2017.

\bibitem[Ranganath and Perotte(2018)]{ranganath2018multiple}
R.~Ranganath and A.~Perotte.
\newblock Multiple causal inference with latent confounding.
\newblock 2018.
\newblock URL \url{https://arxiv.org/abs/1805.08273}.

\bibitem[Sharan et~al.(2017)Sharan, Kakade, Liang, and Valiant]{SharanKLV17}
V.~Sharan, S.~M. Kakade, P.~Liang, and G.~Valiant.
\newblock Learning overcomplete {HMM}s.
\newblock In \emph{Advances in Neural Information Processing Systems}, pages
  940--949, 2017.
\newblock URL \url{https://arxiv.org/abs/1711.02309}.

\bibitem[Shpitser and Pearl(2006)]{shpitserP06}
I.~Shpitser and J.~Pearl.
\newblock Identification of joint interventional distributions in recursive
  semi-{Markovian} causal models.
\newblock In \emph{Proc. 20th {AAAI} Conference on Artifical Intelligence},
  pages 1219--1226, 2006.
\newblock \doi{10.5555/1597348.1597382}.

\bibitem[Spirtes et~al.(2000{\natexlab{a}})Spirtes, Glymour, and
  Scheines]{SpGlS00}
P.~Spirtes, C.~Glymour, and R.~Scheines.
\newblock \emph{Causation, Prediction and Search}.
\newblock MIT Press, second edition, 2000{\natexlab{a}}.

\bibitem[Spirtes et~al.(2000{\natexlab{b}})Spirtes, Glymour, Scheines,
  Kauffman, Aimale, and Wimberly]{spirtes2000constructing}
P.~Spirtes, C.~Glymour, R.~Scheines, S.~Kauffman, V.~Aimale, and F.~Wimberly.
\newblock Constructing {B}ayesian network models of gene expression networks
  from microarray data.
\newblock 2000{\natexlab{b}}.

\bibitem[Tahmasebi et~al.(2018)Tahmasebi, Motahari, and
  Maddah-Ali]{tahmasebi2018identifiability}
B.~Tahmasebi, S.~A. Motahari, and M.~A. Maddah-Ali.
\newblock On the identifiability of finite mixtures of finite product measures.
\newblock (Also in ``On the identifiability of parameters in the population
  stratification problem: A worst-case analysis,'' Proc. ISIT'18 pp.\
  1051-1055), 2018.
\newblock URL \url{https://arxiv.org/abs/1807.05444}.

\bibitem[Thiesson et~al.(1998)Thiesson, Meek, Chickering, and
  Heckerman]{thiesson1998}
B.~Thiesson, C.~Meek, D.~M. Chickering, and D.~Heckerman.
\newblock Learning mixtures of {DAG} models.
\newblock In \emph{Proc. 14th Conf. on Uncertainty in Artificial Intelligence},
  page 504–513, 1998.
\newblock \doi{10.5555/2074094.2074154}.

\bibitem[Titterington et~al.(1985)Titterington, Smith, and Makov]{TSM85}
D.~M. Titterington, A.~F.~M. Smith, and U.~E. Makov.
\newblock \emph{Statistical Analysis of Finite Mixture Distributions}.
\newblock John Wiley and Sons, Inc., 1985.

\bibitem[Wang and Blei(2019)]{wang2018blessings}
Y.~Wang and D.~M. Blei.
\newblock The blessings of multiple causes.
\newblock \emph{Journal of the American Statistical Association}, 114\penalty0
  (528):\penalty0 1574--1596, 2019.
\newblock \doi{10.1080/01621459.2019.1686987}.

\end{thebibliography}
\newpage
\appendix

\section{$k$-MixBND algorithm pseudocode}\label{apx: BND pseudo}

\begin{algorithm} 
    \caption{Alignment}\label{alg:align}
    \DontPrintSemicolon
    \KwIn{ A set of runs $\calA = \{a_0, \ldots, a_\ell\}$ with outputs $ M^a,  \pi^a$ for each $a \in \calA$.  In addition, we have a spanning tree $\mathcal{T} = (\calA, \vec{E})$ and alignment variables $\AV(a_i, a_j) \subseteq \vec{I}^{a^{(i)}} \cap  \vec{I}^{a^{(j)}}$.}
 	\KwOut{$\PP^a_u(\vec{I}^a)$ and $\PP^a(u)$ for each $a \in \calA$ and $u \in [k]$.}
       
       Let $\PP^{a_0}(u) \gets \pi^{a_0}_u$ and $\PP^{a_0}_u(\vec{I}^{a_0}) \gets M^{a_0}_{-,u}$ for an arbitrary choice of $a_0$\;
       
        \For{each edge $(a_i, a_j)$ along a breadth first traversal of $\mathcal{T}$ from $a_0$}{
		Choose some $X_{\AV} \in \AV(a_i, a_j)$.        
        
		Let $q,r$ give the indices for the alignment variable, i.e. $X_{\AV} = X^{a_i}_{q}$ and  $X_{\AV} = X^{a_j}_{r}$.
		
            Find $\sigma$, the permutation on the sources that minimizes $\Norm{ M^{a_i}_{q, -} - \sigma M^{a_j}_{r,-}}_{\infty}$.\;
            
            Assign $\PP^{a_j}(U) \gets \sigma \pi^{a_j}$ and $\PP^{a_j}_u(\vec{I}^a) \gets \sigma M^{a_j}$.\;
        }
\end{algorithm}

\begin{algorithm}
    \caption{Bayesian Unzipping}\label{alg:bayesian-unzipping}
        \KwIn{A collection of runs $\calA$ of size at most $2^{O(\Delta^2)}$ and their aligned output.  For each $V_i$ and assignment to its parents $\pa(V_i)$ there must be some run with $V_i$ in its independent set with parents conditioned to $\pa(V_i)$.}
        \KwOut{$\ePP_u(Y \given \Pa(Y))$}
        Fix a topological ordering on the vertices in $\Vars$, $\Seq{X_1,X_2,\dotsc,X_n}$;
        \For{$i=n,n-1,\dotsc,1$}{
            \For{each assignment $\pa(X_i)$ to $\Pa(X_i)$}{
        	    Let $a$ be a run in $\calA$ with $\pa^a(X_i) = \pa(X_i)$.\;
                \For{$u=1,\dotsc,k$}{
                    \For{$b=0,1$}{
                        \uIf{$X_i$ is a bottom vertex}{
                           Set $\ePP_u(x_i^b\given \pa^a(X_i))\gets \ePP_u^a(x_i^b)$.\;
                        }
                        \Else{
                            Set $\ePP_u(x_i^b\given \pa(X_i)) \gets \frac{\ePP_u^a(x_i^b)\ePP_u(\ch^a(X_i)\given \ttop^a(X_i),x_i^{1-b})}{\ePP_u^a(x_i^b)\ePP_u(\ch^a(X_i)\given \ttop^a(X_i),x_i^{1-b})+\ePP_u^a(x_i^{1-b})\ePP_u(\ch^a(X_i)\given \ttop^a(X_i),x_i^b)}$;
                        }
                    }
                }
            }
        }
\end{algorithm}

\begin{algorithm}
    \caption{The $k$-MixBND algorithm}\label{alg:k-bnd} 
        \KwIn{ A good collection of runs $\calA$ of size at most $2^{O(\Delta^2)}$.}
        \KwOut{ $\ePP_u(Y \given \Pa(Y))$ and $\ePP(U)$.}
        Estimate $\ePP^a(\II^a)$ for all runs $a\in \calA$. \;   
        
        \For{each run $a\in \calA$}{            
           Set $\Paren{M^a, \pi^a}\gets \textsc{LearnProductMixture}(\mathcal{P}^{a}(\vec{I}^a))$\;
        }
        
        Run Algorithm~\ref{alg:align} to align the sources in the output of all the runs in $\calA$.\;
        
        Run Algorithm~\ref{alg:bayesian-unzipping} to unzip the parameters.\;
        
        Fix any run $a\in \calA$.\;
        
        \For{$u = 1, \ldots, k$}{
         	Set $\ePP(u) \gets \frac{\ePP(u \given \mcond^a)\ePP(\mcond^a)}{\ePP_{u}(\mcond^a)}$.\;
        }
\end{algorithm}
\newpage

\section{Finding good collections of runs} \label{apx: good collections}
In this section we prove Theorem~\ref{thm:main result on good collection} and give a few different sufficient conditions for the existence of a good collection of runs.

\subsection{Constructing a good collection of runs from $N_{\text{mp}}$ centers}\label{sec:centers to good collection}
One sufficient condition for the existence of a good collection of runs is the presence of $N_{\text{mp}}$ variables with disjoint Markov boundaries. We will now present a good collection of runs for this case.
\begin{lemma} Given a set $\Centers = \Set{X_1,\dotsc, X_{N_{\text{mp}}}} \subseteq \Vars$ of $N_{\text{mp}}$ variables with depth at most $3N_{\text{mp}}$ and disjoint Markov boundaries, Algorithm~\ref{alg:good runs} finds a good collection of runs $\calA$, satisfying $\Card{\calA} = O(2^\gamma n)$ where $\gamma = \max_{V\in \Vars} \Card{\Mb(V)}$,.\label{lem:good collection}
\end{lemma}

 \begin{algorithm}
     \caption{Building a good collection of runs}\label{alg:good runs}
         \KwIn{Vertices $\Centers=\Set{X_1,\dotsc,X_{3k-3}}\subseteq \Vars$ having disjoint Markov boundaries with maximum depth $3N_{\text{mp}}$.}
         \KwOut{A good collection of runs $\calA$.}
          Let $a_0$ be a run with $\II^{a_0} = \Set{X_1,\dotsc,X_{3k-3}}$ and $\MCond^{a_0}$ chosen arbitrarily.\;
         
         Set $\calA \gets \Set{a_0}$.\;
        
         \For{$i=1,\dotsc,3k-3$}{
          $\calA\gets \calA\cup \Set{a_0[\mb(X_i)]:\mb(X_i)\in \Set{0,1}^{\Mb(X_i)}}$.\;
         \For{$Y\in \Vars \setminus \Centers$}{
         \If{$Y \in \Mb(X)$ for some $X \in \XX$}{
          $\II^a \gets \Centers \cup \{Y\} \setminus \{X_i\}$\;
          }
        \Else{
          $\II^a \gets \Centers \cup \{Y\}$\;
         }
         \For{$\pa(Y) \in \{0, 1\}^{\abs{\Pa(Y)}}$}{
         \If{$Y \in \An(\II^a - Y)$}{
              $\MCond^a \gets \Mb(\II^a)$\;
         }\Else{
             $\MCond^a \gets \Mb(\II^a \setminus \{Y\}) \cup \Pa(Y)$\;
         }
        $\MCond^a \gets \Mb(\II^a)$\;
        
         $f^a(\Pa(Y)) \gets \pa(Y)$\;
        
         $\DEFAULTS \gets \MCond^{a_0} \cap \MCond^a \setminus \Pa(Y)$\;       
         
         $f^a(\DEFAULTS) \gets f^{a_0}(\DEFAULTS)$\;
            
          $f^a(\MCond^a \setminus \MCond^{a_0} \setminus  \Pa(Y))$ are chosen arbitrarily.\;
             
         $\calA \gets \calA \cup \{a\}$, where $a$ is given by $a = (\II^a, f^a)$.\;
         }}} 
 \end{algorithm}
 
\begin{claim}
    By construction, $\calA_C$ covers $\Centers$ and is alignable. \label{clm:central runs}
\end{claim}    
    
\begin{claim} $A_Y$ covers $\Vars\setminus \Centers$ and each run in $\calA_Y$ can be aligned with some run in $\calA_C$.  \label{clm:remaining runs}
\end{claim}

\begin{proof} The fact that $\calA_Y$ covers $\Vars\setminus \Centers$ follows immediately from $\calA_Y$ containing a run assigning for independent variable $Y \in \Vars\setminus \Centers$ each possible assignment $\pa(Y)$. Fix any run $a\in \calA_Y$ with $\II^a = \Centers - X_i + Y$ or $\II^a = \Centers + Y$ depending on whether $Y$ overlaps with $\Mb(\XX)$.  Now if $\Mb(Y)\cap \Mb(X_j) = \emptyset$ for any $j\neq i$, $a$ and $a_0$ are aligned at $X_j$. If instead $\Mb(Y)\cap \Mb(X_j) \neq \emptyset$ for all $j\neq i$, pick any $j$ and consider the central run $a_0[\mb^a(X_j)] \in \calA_C$. Clearly, $a$ and $a_0[\mb^a(X_j)]$ are aligned at $X_j$. In either case, we've aligned $a$ to a run in $\calA_C$. \end{proof}

\begin{claim} Every run $a\in \calA$ is at least $N_{\text{mp}}$-independent. \label{clm:2k independent}
\end{claim}

\begin{proof}[Proof of Lemma~\ref{lem:good collection}] This follows immediately from Claims~\ref{clm:central runs},~\ref{clm:remaining runs},~and~\ref{clm:2k independent}
\end{proof}

\suppress{\begin{lemma}\label{lem: enough vertices for 2k centers} Let $\G=(\Vars, \vec{E})$ be a DAG with $n=\Card{\Vars}$, and let $\gamma \coloneqq \gamma(G)$ be as in Definition~\ref{def:gamma}.
    Then if $n \geq (3k-3)(\gamma^2 - \gamma + 1)$ we can find a set of $(3k-3)$ centers for $\G$ with depth at most $9k$.
\end{lemma}
\begin{proof} Let $n \geq (3k-3)(\gamma^2 - \gamma + 1)$.
    Fix any source vertex $Y\in \Vars$. Then
    \begin{align*}
        \Card{\Mb(\Mb(Y)) \cup \Mb(Y)\cup \Set{Y}}
         & = \Card{\cup_{V\in \Mb(Y)}\Mb(V)}                              \\
         & \leq \Card{\Mb(Y)}\Paren{\max_{V\in \Mb(Y)} \Card{\Mb(V)}-1}+1 \\
         & \leq \gamma^2 - \gamma + 1
    \end{align*} so we can remove a source vertex $Y\in \Vars$ along with $\Mb(Y)\cup \Mb(\Mb(Y))$ from $\G$ at least $(3k-3)$ times before no vertices remain. The $i$th vertex removed will have depth at most $3i$. Now by Observation~\ref{obs:mb mb}, if any subsequent vertex $Z$ has a Markov boundary that overlaps $\Mb(Y)$, it must be the case that $Z\in \Mb(\Mb(Y))$, which is not possible since $\Mb(\Mb(Y))$ was removed from $\G$. Thus, we can find $(3k-3)$ centers at depth at most $9k$.
\end{proof}}
\subsection{Degree bounds}
We can ensure that $N_{\text{mp}}$ centers can be found on certain degree-bounded graphs, which in turn bound $\gamma$.  Let $\din$ upper bound on the in-degree of any vertex in $\G$ and let $\dout$ upper bound the out-degree.  Then
\begin{equation*}
    \gamma \leq \din + \dout + \dout(\din-1) = \din + \dout\din.
\end{equation*}
If we have a bound $\Delta$ on the degree of the undirected skeleton of $\G$, we get that
\begin{equation*}
    \gamma \leq \Delta(\Delta - 1) = \Delta^2-\Delta.
\end{equation*}
\begin{corollary}\label{cor:centers from degree bound}
    If either of the following conditions hold, we can find $N_{\text{mp}}$ centers for $\G$ with depth at most $3 N_{\text{mp}}$:
    \begin{enumerate}
        \item $n \geq N_{\text{mp}}(\din^2 + 2\dout\din +\dout^2\din^2-\din - \dout +1) = N_{\text{mp}}\cdot O(\dout^2\din^2)$.
        \item $n \geq N_{\text{mp}}(\Delta^4 - 2\Delta^3+\Delta + 1) = N_{\text{mp}}\cdot O(\Delta^4)$.
    \end{enumerate}
\end{corollary}

\begin{proof}[Proof of Lemma~\ref{thm:main result on good collection}]
    This follows immediately from Corollary~\ref{cor:centers from degree bound} and Lemma~\ref{lem:good collection}.
\end{proof}

\section{Analyzing Bayesian Unzipping}\label{apx: BND analysis}
 \label{sec:recovering parameters}
In this section, we will analyze the numerical stability of the Bayesian unzipping process.

\begin{lemma} Let $a$ be a run with $Y\in \II^a$. Then we can compute $\PP_u(y^b\given \pa^a(Y))$ for $b\in \Set{0,1}$ as follows:
    \begin{enumerate}
        \item If $Y\in \Bottom(\II^a)$, then \[ \PP_u(y^b\given \pa^a(Y)) = \PP_u^a(y^b). \]
        \item If $Y\notin \Bottom(\II^a)$, then \begin{equation}
                  \PP_u(y^b\given \pa^a(y)) = \frac{\PP_u^a(y^b)\PP_u(\ch^a(y)\given \ttop^a(y),y^{1-b})}{\PP_u^a(y^b)\PP_u(\ch^a(y)\given \ttop^a(y),y^{1-b})+\PP_u^a(y^{1-b})\PP_u(\ch^a(y)\given \ttop^a(y),y^b)}\label{eq:parameter recovery}\end{equation}
    \end{enumerate} \label{lem:recovery equations}
\end{lemma}
\begin{proof}
    In the following we'll fix $\ch^a \coloneqq \ch^a(Y)$, $\mb^a \coloneqq \mb^a(Y)$, $\pa^a \coloneqq \pa^a(Y)$, and $\ttop^a \coloneqq \ttop^a(Y)$.
    If $Y\in \Bottom(\II^a)$, then $\PP_u^a(y^b) = \PP_u(y^b\given \pa^a(Y))$, so the claim is trivially true.
    If $Y\notin \Bottom(\II^a)$, then
    \begin{align*}
        \PP_u^a(y^b)  = \PP_u^a(y^b\given \mb^a)
        = \frac{\PP_u(y^b,\mb^a)}{\PP_u(\mb^a)}
        = \frac{\PP_u(y^b,\mb^a)}{\sum_{b' = 0,1}\ePP_u(y^{b'}, \mb^a)}
    \end{align*} and we can expand $\PP_u(y^b, \mb^a)$ as
    \[ \PP_u(y^b, \mb^a) = \PP_u(y^b\given \pa^a)\PP_u(\ch^a\given \ttop^a, y^b) \PP_u(\ttop^a). \] Substituting this back into the preceding equation, we obtain
    \begin{align*}
        \PP_u^a(y^b)
         & =  \frac{\PP_u(y^b\given \pa^a)\PP_u(\ch^a\given \ttop^a,y^b)\PP_u(\ttop^a)}{\sum_{b'\in \Set{0,1}}\PP_u(y^{b'}\given \pa^a(Y))\PP_u(\ch^a\given \ttop^a, y^{b'}) \PP_u(\ttop^a)} \\
         & = \frac{\PP_u(y^b\given \pa^a)\PP(\ch^a\given \ttop^a,y^b)}{\sum_{b'\in \Set{0,1}}\PP_u(y^{b'}\given \pa^a)\PP_u(\ch^a\given \ttop^a y^{b'})}.
    \end{align*}

    We now multiply both sides by the denominator (which is non-zero since all terms are strictly positive by assumption) and then simplify to obtain
    \[ \PP_u(y^b\given \pa^a)\Paren{\PP_u^a(y^{1-b})\PP(\ch^a\given \ttop^a,y^b)} + \PP_u(y^{b}\given \pa^a)\Paren{-\PP_u^a(y^{1-b})\PP_u(\ch^a\given \ttop^a,y^{1-b})} = 0. \]
    When augmented with the equation \[ \PP_u(y^b\given \pa^a) + \PP_u(y^{1-b}\given \pa^a) = 1\] we have a system of two equations in $\PP_u(y^b\given \pa^a),\PP_u(y^{1-b}\given \pa^a)$ which is non-singular whenever
    \[ \Paren{\PP_u^a(y^{1-b})\PP(\ch^a\given \ttop^a,y^b)} \neq \Paren{-\PP_u^a(y^{1-b})\PP_u(\ch^a\given \ttop^a,y^{1-b})}.\] Since all probabilities occurring are strictly positive, this will always be the case and we can solve the system for $\PP_u(y^b\given \pa^a)$. The resulting equation is
    \[ \PP_u(y^b\given \pa^a) = \frac{\PP_u^a(y^b)\PP(\ch^a\given \ttop^a,y^{1-b})}{\PP_u^a(y^b)\PP(\ch^a\given \ttop^a,y^{1-b})+\PP_u^a(y^{1-b})\PP(\ch^a\given \ttop^a,y^{b})}, \] proving the claim.
\end{proof}

The following standard inequalities will be useful throughout the analysis:
\begin{observation}
    \begin{align*} \frac{1+x}{1-x} & \leq 1+3x\quad\text{and}    & \quad 1-3x    & \leq \frac{1-x}{1+x} \quad & \text{for $x\in (0, 1/4)$;}                 \\
               1-2rx           & \leq (1-x)^r\quad\text{and} & \quad (1+x)^r & \leq 1+2rx          \quad  & \text{for $x\in (0,1),r\geq 1, rx \leq 1$.}
    \end{align*} \label{obs:inequalities}
\end{observation}

\begin{lemma}\label{lem:canrecover_inductive} Given access to $\Set{\ePP_u^a(Y)}_{a\in \calA, j\in [k]}$ for a good collection of runs $\calA$ from a distribution $\PP$ satisfying $\Abs{\ePP_u^a(y^b) - \PP_u^a(y^b)}/\PP_u^a(y^b) \leq \eps$ for all $Y\in \Vars$, $b\in \Set{0,1}$, $u\in [k]$ for $\eps$ sufficiently small (and the estimated probabilities used as input to $k$-MixProd oracles), the procedure in Algorithm~\ref{alg:bayesian-unzipping} will output $\ePP_{u}(y^b\given \Pa(Y))$ and $\ePP(u)$ for all $Y\in \Vars$, $b\in \Set{0,1}, u\in [k]$ satisfying
    \[ \frac{\Abs{\ePP_u(y^b\given \pa(Y)) - \PP_u(y^b\given \pa(Y))}}{\PP_u(y^b\given \pa(Y))} \leq (6(\Delta + 1))^{\ell}\eps \] where $\ell$ is the distance from $Y$ to the nearest bottom vertex if $Y$ doesn't appear as a bottom vertex and is $0$ if $Y$ is a bottom vertex. \end{lemma}

\begin{proof}
    We fix a topological ordering on $\VV$, let's say $X_1,X_2,\dotsc, X_n$. Now starting with the last vertex in topological order, $X_n$, and proceeding in decreasing order, we'll compute $\ePP_u(X_i\given \Pa(X_i))$. For bottom vertices, we set $\ePP_u(x_i^b\given \pa(X_i)) = \ePP_u^a(x_i^b)$ for some run $a$ with $\pa(X_i) = \pa^a(X_i)$. It immediately follows that $P_u(X_i\given \Pa(X_i))$ satisfies the desired relative error bound. Inductively, assume we've already computed $\ePP_u(V_m\given \Pa(V_m))$ for all $m > i$ satisfying the stated bounds, and that we also have access to $\ePP^a_u(V_i) = \ePP_u(V_i \given \MCond^a)$ for a subset of the runs $a \in \calA$ that cover $V_i$. To streamline notation, let $Y \coloneqq V_i$. Now fix a run $a\in \calA$ with $Y\in \II^a$.

    Now we can bound each term above and below using the inductive hypothesis to get the desired result. In particular, we know that
    \begin{equation*}
        (1-\eps)\PP_u^a(y^b)                                        \leq \ePP_u^a(y^b)                \leq (1+\eps)\PP_u^a(y^b)                                                                        \end{equation*} for $b\in \Set{0,1}$ and \begin{equation*}
        (1-(6(\Delta + 1))^{\ell-1}\eps)\PP_u(v^b\given \pa^a(v^b)) \leq \ePP_u(v^b\given \pa^a(v^b))  \leq (1+(6(\Delta + 1))^{\ell-1}\eps) \PP_u(v^b\given \pa^a(v^b))\end{equation*}
    for all $V\in \Ch(Y)$ and $b\in \Set{0,1}$ which implies that both the numerator and denominator of \eqref{eq:parameter recovery} are within a factor of $(1\pm \eps)(1\pm (6(\Delta+1))^{\ell-1} \eps)^\Delta$ of the correct value for those terms. It immediately follows that $\ePP_u(y^b\given \pa^a(Y))$ is bounded by
    \begin{align*}\PP_u(y^b \given \pa^a(Y)) \frac{1-\eps}{1+\eps} \Paren{\frac{1-(6(\Delta+1))^{\ell -1}\eps}{1+(6(\Delta+1))^{\ell -1}\eps}} &\leq \ePP_u(\yone \given \pa^a(Y)) \\&\leq \PP_u(\yone \given \pa^a(Y)) \frac{1+\eps}{1-\eps} \Paren{\frac{1+(6(\Delta+1))^{\ell -1}\eps}{1-(6(\Delta+1))^{\ell -1}\eps}}. \end{align*}

    We now use the inequalities from Observation~\ref{obs:inequalities} to simplify the bounds as follows:
    \begin{align*}\ePP_u(y^b \given \pa^a(Y))
         & \leq \PP_u(y^b \given \pa^a(Y)) \frac{1+\eps}{1-\eps} \Paren{\frac{1+(6(\Delta+1))^{\ell -1}\eps}{1-(6(\Delta+1))^{\ell -1}\eps}}^{\Delta} \\
         & \leq \PP_u(y^b \given \pa^a(Y))(1+3\eps)(1+3(6(\Delta+1))^{\ell -1}\eps)^{\Delta}                                                          \\
         & \leq \PP_u(y^b \given \pa^a(Y))(1+3(6(\Delta+1))^{\ell -1}\eps)^{\Delta+1}                                                                 \\
         & \leq \PP_u(y^b \given \pa^a(Y))(1+2\cdot 3(\Delta + 1)(6(\Delta+1))^{\ell -1}\eps)                                                         \\
         & \leq \PP_u(y^b \given \pa^a(Y))(1+(6(\Delta+1))^{\ell}\eps).                                                                               \\
    \end{align*} The lower bound is analogous.
\end{proof}

\begin{lemma}\label{lem:k-bnd final error blowup} Given access to $\Set{\ePP_u^a(Y)}_{a\in \calA, j\in [k]}$ for a good collection of runs $\calA$ from a distribution $\PP$ satisfying $\Abs{\ePP_u^a(y^b) - \PP_u^a(y^b)}/\PP_u^a(y^b) \leq \eps$ for all $Y\in \II^a$, $b\in \Set{0,1}$, $u\in [k]$ for $\eps$ sufficiently small, the procedure in Algorithm~\ref{alg:bayesian-unzipping} will output $\ePP_{u}(y^b\given \Pa(Y))$ and $\ePP(u)$ for all $Y\in \Vars$, $b\in \Set{0,1}, u\in [k]$ satisfying
    \[ \frac{\Abs{\ePP_u(y^b\given \pa(Y)) - \PP_u(y^b\given \pa(Y))}}{\PP_u(y^b\given \pa(Y))} \leq (6(\Delta+1))^{3 N_{\text{mp}}}\eps \] and \[ \frac{\Abs{\ePP(u^j) - \PP(u^j)}}{\PP(u^j)} \leq 5\Delta^2\eps. \]

    Since $\PP_u(y^b \given \pa(Y)), \PP(u^j) \leq 1$, we also have that \[ \Abs{\ePP_u(y^b\given \pa(Y)) - \PP_u(y^b\given \pa(Y))} \leq (6(\Delta+1))^{9k}\eps\quad \text{and}\quad\Abs{\ePP(u^j) - \PP(u^j)} \leq 5\Delta^2\eps. \]\end{lemma}
\begin{proof}
    The error bound on $\ePP_u(y^b\given \pa(Y))$ follows from Lemma~\ref{lem:canrecover_inductive} above and the depth bound of $9k$ on non-bottom vertices.

    For the error bound on $\ePP(u^j)$ we write
    \begin{align*} \ePP(u^j) = \frac{\ePP(u^j\given \mcond^a)\ePP(\mcond^a)}{\ePP(\mcond^a\given u^j)}
    \end{align*} and analyze each term separately. First, define $Z \coloneqq \An(\MCond^a)\setminus \MCond^a$ to be all ancestors of vertices conditioned upon in $a$ minus $\MCond^a$. Fix a topological order on the vertices in $Z \cup \MCond^a$, $\Seq{X_1,\dotsc, X_m}$. We can write
    \begin{align*}
        \ePP(\mcond^a\given u^j)
          =& \sum_{z\in \Set{0,1}^{Z}}\prod_{i=1}^m \ePP(z(X_i)\given z(X_1,\dotsc, X_{i-1}),\mcond^a)                                                                                 \\
          =& \sum_{z\in \Set{0,1}^{Z}}\prod_{X_i \in Z\cup \MCond^a} \ePP(z(X_i)\given z(\Pa(X_i)), \pa^a(X_i))                                                                        \\
          =& \sum_{z\in \Set{0,1}^{Z}}\Paren{\prod_{X_i \in Z} \ePP(z(X_i)\given z(\Pa(X_i)), \pa^a(X_i))}\\
          &\Paren{\prod_{X_i \in \MCond^a} \ePP(z(X_i)\given z(\Pa(X_i)), \pa^a(X_i))}.
    \end{align*} 
    Now let $C_z$ denote the value in the first grouped product for a given $z\in \Set{0,1}^Z$. Since $\sum_{z\in \Set{0,1}^Z} C_z = 1$, we can bound the error in the result by the error in the second grouped product:
    \[ (1-\eps)^{\Card{\MCond^a}}\PP(\mcond^a\given u^j) \leq \ePP(\mcond^a\given u^j) \leq (1+\eps)^{\Card{\MCond^a}}\PP(\mcond^a\given u^j). \] Finally using the bound $\Card{\Cond^a} < 2\Delta^2$ and the inequalities from Observation~\ref{obs:inequalities} we obtain
    \[ (1-2\Delta^2\eps)\PP(\mcond^a\given u^j) \leq \ePP(\mcond^a\given u^j) \leq (1+2\Delta^2\eps)\PP(\mcond^a\given u^j). \] By assumption $\ePP(u^j\given \cond^a) \in (1\pm \eps)\PP(u^j\given \cond^a)$; $\ePP(\cond^a)$ is also known up to multiplicative accuracy $\eps$ since the sampling needed to estimate the distributions that are input to $k$-MixProd oracles far exceed that needed to get an $\eps$ estimate of this quantity.
    Thus, the resulting bound on $\ePP(u^j)$ is
    \[ (1-2\Delta^2\eps)(1-3\eps)\PP(u^j) \leq \ePP(\mcond^a\given u^j) \leq (1+2\Delta^2\eps)(1+3\eps)\PP(\mcond^a\given u^j) \] which can be simplified to
    \[ (1-2\Delta^2\eps-3\eps)\PP(u^j) \leq \ePP(\mcond^a\given u^j) \leq (1+2\Delta^2\eps+3\eps)\PP(\mcond^a\given u^j) \] using $(1-x)(1-y) \geq (1-x-y)$ for $x,y\in [0,1)$ and $2\Delta^2\eps + 3\eps \leq 5\Delta^2\eps$ for $\Delta \geq 1$.
\end{proof}

\subsection{Mixtures of paths}\label{apx:paths}

\newcommand{\odd}{\textsf{ODD}}
\newcommand{\even}{\textsf{EVEN}}
\newcommand{\link}{\textsf{LINK}}
\newcommand{\tail}{\textsf{TAIL}}
Special cases do not require the condition of $N_{\text{mp}}$ disjoint Markov boundaries. One of these special cases is a mixture of paths $v_1 \rightarrow v_2 \rightarrow v_3 \rightarrow \dotsb v_{n}$ over $k$. We will will give a good set of runs for $n\geq 2 N_{\text{mp}}$.

First,  we will specify three default runs,  which we will call $\odd$,  $\even$,  and $\link$.
\begin{definition}[$\odd$ default run for Markov chains]
    Run $\odd$ is specified by an independence set of vertices with odd indices $\II^{\odd} = \{V_1, V_3,  V_5, \ldots, V_{2N_{\text{mp}}-1}\}$. The conditioning set is given by the evenly indexed vertices $\MCond^{\odd} = \{V_2, V_4,  V_6, \ldots, V_{2N_{\text{mp}}-2} \}$. $f^\odd$ may be chosen arbitrarily, but for simplicity we will give $f^{\odd} (V_i) = 0$ for all $V_i \in \MCond^{\odd}$.
\end{definition}

\begin{definition}[$\even$ default run for Markov chains]
    Run $\even$ is defined the same way for evenly indexed vertices. That is,  $\II^{\even} = \{V_2, V_4,  V_6, \ldots,V_{2N_{\text{mp}}} \}$, $\MCond^{\even} = \{V_1, V_3,  V_5, \ldots, V_{2N_{\text{mp}}-1} \}$, and $f^{\even} (v_i) = 0$ for all $V_i \in \MCond^{\even}$.
\end{definition}

\begin{definition}[$\link$ default run for Markov chains]
    Run $\link$ is part evenly indexed vertices and part oddly indexed vertices.  $\II^{\link} = \II^{\even} \cup \{V_1\} \setminus \{V_2\}$.  Similarly, we condition on the complement $\MCond^\link = \{V_2, V_3,  V_5, \ldots, V_{2N_{\text{mp}}-1} \}$ being $0$.
\end{definition}

\begin{claim}
    If $n \geq 2N_{\text{mp}}$, then $\odd$,  $\even$, and $\link$ are well-formed runs.
\end{claim}

\begin{claim}
    $\odd$ and $\link$ are aligned at $V_{1}$.  $\even$ and $\link$ are aligned at evenly indexed vertices beyond $V_2$.
\end{claim}

Now,  we enumerate a set of runs that cover all possible entries to parents of vertices.

\begin{definition}
    Let $\odd[v_i]$ denote a run on $\II^{\odd[V_i]} = \II^{\odd}$, $\MCond^{\odd[V_i]} = \MCond^{\odd}$ with $f^{\odd[V_i]}(V_j) = 1$ if $i=j$ and $0$ if $i\neq j$. Similarly define $\even[V_i]$ to be a run on $\II^{\even[V_i]} = \II^{\even}$, $\MCond^{\even[V_i]} = \MCond^{\even}$  with $f^{\even[V_i]}(v_j) = 1$ if $i=j$ and $0$ if $i\neq j$.
\end{definition}

\begin{claim}
    Any run $\odd[V_i]$ is aligned with $\odd$ at $V_j$ for $j < i-1$ or $j > i+1$. Similarly,  any run $\even[V_i]$ is aligned with $\even$ at $V_j$ for $j < i-1$ or $j > i+1$.
\end{claim}

\begin{definition}
 Let $\tail^b[V_i]$ for $i >2N_{\text{mp}}$ give runs which have 
 \begin{align*}
 \II^{\tail[V_i]} &= \{V_1, V_3, V_5, \ldots, V_{2N_{\text{mp}}-1}, V_{i}\}\\
 \MCond^{\tail[V_i]} &= \{V_2, V_4, \ldots, V_{2N_{\text{mp}}-2},V_{i-1}\}\\
 \end{align*}
  with
  \[
  f^{\tail[V_i]}(V_j) = \begin{cases} 
  	0 & \text{ if } i-1 \neq j\\
  	b & \text{ if } i-1 = j
  \end{cases}
  \]
\end{definition}
\begin{claim}
	Any run $\link[V_i]$ is aligned with $\odd$ at $V_j$ for $j < 2N_{\text{mp}}$ with even $j$.
\end{claim}

\begin{claim}
    The set \[\left\{\odd[V_i]\;:\; v_i \in \II^{\odd}\right\} \cup \left\{\even[V_j]\;:\; V_j \in \II^{\even}\right\} \cup \left\{\tail^0[V_j], \tail^1[V_j]\;:\; j>6k-6\right\}\] covers $\Vars$.
\end{claim}
\begin{proof}
For all $V_i$ with $i > 6k-6$ we have $\link^0[V_i]$ and $\link^1[V_i]$ to cover both possible assignments to $V_{i-1}$.    
    
    Consider $V_i$ with $i \leq 2N_{\text{mp}}$.  If $i>1$ is odd, then $f^{\even[V_i]}(V_{i-1}) = 0$ and $f^{\even}(V_{i-1}) = 1$, so all possible assignments to the parent of $V_i$ are covered.  Similarly, if $i$ is even, then $f^{\odd[V_i]}(V_{i-1}) = 0$ and $f^{\odd}(V_{i-1}) = 1$, so all possible assignments to the parent of $V_i$ are covered.  Finally, if $i=1$, then $V_i$ has no parents.
\end{proof}

We give the following example to illustrate this construction.  Consider $k=2$ and $n=8$.
We will express a run $a$ as sequences of values $\texttt{0}$, $\texttt{1}$ or $\texttt{*}$.
A $\texttt{*}$ in the $i$th location indicates $V_i \in \II^a$ and a $\texttt{0}$ or $\texttt{1}$ indicates $f^a(V_i) = 0$ or $f^a(V_i) = 1$, respectively.  Finally, $\texttt{-}$ indicates that the variable is not conditioned on and not in the independent set (not in $\MCond$ and not in $\II$).
The default runs are:
\begin{align*}
    \even  & = \texttt{0*0*0*--} \\
    \odd & = \texttt{*0*0**--} \\
    \link & = \texttt{*00*0*--} \\
\end{align*}
In addition to these runs, we have:
\begin{align*}
    \odd[V_1]  & = \texttt{1*0*0*--} \\
    \odd[V_3]  & = \texttt{0*1*0*--} \\
    \odd[V_5]  & = \texttt{0*0*1*--} \\
    \even[V_2] & = \texttt{*1*0*---} \\
    \even[V_4] & = \texttt{*0*1*---} \\
    \even[V_6] & = \texttt{*0*0*---} \\
    \tail^0[V_7] & = \texttt{*0*0-0*-} \\
    \tail^1[V_7] & = \texttt{*0*0-1*-} \\
    \tail^0[V_8] & = \texttt{*0*0--0*} \\
    \tail^1[V_8] & = \texttt{*0*0--1*} \\
\end{align*}

The construction of a good set of runs given does not use disjoint Markov boundaries, yielding greater efficiency than our more general algorithm for degree bounded graphs.

While similar in name and structure,  this setting is different from that of hidden Markov models.  Hidden Markov models are Bayesian
networks consisting of a chain of unobserved variables, each affecting a unique observed variable, with no causal relations
among the observed variables.
For instance, in~\citet{GKV16}, all chains in the mixture have the same state space. The sampling
process selects a chain and a starting state from the mixture distribution, then generates a short observable path through the
chain. In their model, under some conditions, a sufficiently large collection of paths of length $3$ suffices to recover the
parameters of the mixture, using spectral methods. A different model is considered in~\citet{BGK04}. There, the constituents of
the mixture have disjoint state spaces and the observation is a long sequence of interleaved paths in the separate chains.
The primary challenge is to cluster the observable states into the $k$ constituents. This model can be viewed as a special
case of the hidden Markov model problem.

The natural extension of $\even$ and $\odd$ vertices is a two-coloring on a tree which denotes sets that are of even or odd distance from the root.  One can construct a good set of runs for learning mixtures of trees of size $n \geq 2N_{\text{mp}}$ that is very similar to the one given for mixtures of paths.

\subsection{Larger alphabets for mixtures of DAGs}
\label{apx: large alphabets for DAGs}
The simplest reduction for larger alphabets is to replace each vertex with a clique of $d$ binary vertices which represent the value of the nonbinary vertex.  The pseudocode for this process is given in Algorithm~\ref{alg:larger_alphabet_clique_reduction}.

 \begin{algorithm}
     \caption{Reduction for larger alphabets}
     \label{alg:larger_alphabet_clique_reduction} 
         \KwIn {A DAG $(\vec{V}, \mathcal{E})$ on $n$ variables $V_1, \ldots, V_n \in [d]$.}
         \KwOut {A DAG $(\mathcal{W}, \mathcal{E}_W)$ on $dn$ binary variables $W_1^1,  \ldots W_1^d, \ldots W_n^1, \ldots W_n^d \in \{0, 1\}$.}
        Start with $ \mathcal{E}_W \gets \emptyset$\;
         \For{each vertex $i \in [n]$}{
         Form a clique among $W_i^1,  \ldots W_i^d$ by adding directed edges $(W_i^a, W_i^b)$ to $\mathcal{E}_W$ for all pairs $a < b$ with $a, b \in [d]$.\;
         }
         \For{each directed edge $(V_i, V_j) \in \mathcal{E}$}{
        Add $\{W_i^1, \ldots, W_i^d\} \times \{W_j^1, \ldots, W_j^d\}$ to $\mathcal{E}_W$.\;
        }
 \end{algorithm}
\begin{observation}
    The maximum degree of $(\mathcal{W}, \mathcal{E}_W)$ outputted by Algorithm~\ref{alg:larger_alphabet_clique_reduction} is now at least $\Delta \geq d$.
\end{observation}

\begin{observation}
    The number of vertices $|\mathcal{W}| = nd$.
\end{observation}

We now give the function that translates data from the original graph to the new graph.
\begin{definition}[One-hot encoding]
    We define $\chi(v) = (\mathbb{1}_0(v),\ldots,  \mathbb{1}_{d-1}(v))$ to give the one-hot encoding of the value of a variable $V$.
\end{definition}

This allows us to give the full algorithm.
\begin{algorithm} 
 \caption{DAG reduction for larger alphabets}\label{alg:larger_alphabet_parameter_reduction}
 \KwIn{A DAG $(\vec{V}, \mathcal{E})$ on $n$ variables $V_1, \ldots, V_n \in [d]$.  And data with entries of the form $(v_1, v_2, \ldots, v_n)$.}
 \KwOut{Parameters $\PP_u(v_i \given \Pa(V_i))$ for $i \in [n]$.}
Use Algorithm~\ref{alg:larger_alphabet_clique_reduction} to create a larger DAG on binary variables,  called $\G'$.\;
        
One-hot encode data on each $V_i$ into binary variables $\xi(V) = (W^1_i, \ldots, W^d_i)$.\;
        
Run the Mixture of DAGs algorithm for $\G'$ on the one-hot encoded data.\;

\For{each parameter $\PP_u(v_i \given \Pa(V_i))$}{
 Let $Z$ indicate zero assignments to $W^b_i$ for $b \neq v_i$.\;
        
One-hot encode each parent $V_j \in \Pa(V_i)$ to obtain assignments to $\Pa(W^{v_i}_i) \setminus \Set{w^{b} \;:\; b \neq v_i}$, denoted $\pa^{\OH}(W^{v_i}_i)$.\;
        
Assign $\PP_u(v_i \given \Pa(V_i)) = \PP_u(W_i^{ v_i} = 1 \given Z, \pa^{\OH}(W^{v_i}_i))$.\;
}
 \end{algorithm}

\begin{theorem}
    If $\PP(\Vars) =\sum_{u}\PP_u(\Vars)\PP(u) $ is a mixture of $k$ distributions over a DAG $\G=(\Vars, \vec{E})$ with $\Vars \in \set{0, \ldots, d-1}$ and there exists a set of $N_{\text{mp}}$ centers, then there is an algorithm to compute estimates of all parameters to accuracy $\eps$.  If $D = \max(d, \Delta)$ then the algorithm uses $O(nd 2^{D^2})$ calls to the $k$-MixProd oracle.
\end{theorem}
\begin{proof}
    The algorithm uses the reduction in Algorithm~\ref{alg:larger_alphabet_parameter_reduction}. This involves running the algorithm on a larger graph with $nd$ vertices and $\Delta \geq d$, which modifies the run-time and sample complexity as given.
\end{proof}

\end{document}